%% file: main.tex
\documentclass{article}

\usepackage{iclr2021_conference,times}

\iclrfinalcopy

\usepackage[utf8]{inputenc} 
\usepackage[T1]{fontenc}    
\usepackage{url}            
\usepackage{booktabs}       
\usepackage{amsfonts}       
\usepackage{nicefrac}       
\usepackage{microtype}      
\usepackage{listings}
\usepackage{url}

\usepackage{graphicx}
\usepackage{subfigure}
\usepackage{booktabs} 
\usepackage{caption}

\usepackage{framed}
\usepackage{amssymb}
\usepackage{amsfonts}
\usepackage{mathrsfs}

\usepackage{mathtools}
\usepackage{array}
\usepackage{amsthm}
\usepackage{verbatim} 
\usepackage{enumerate}
\usepackage{bbm}
\usepackage{commath}
\usepackage{wrapfig}
\usepackage{amsbsy}
\usepackage{float}
\usepackage{amsmath}
\usepackage{algorithm}
\usepackage{tabularx}
\usepackage{changepage}

\input{math_commands.tex}

\newcommand{\xm}{x^{-}}
\newcommand{\xp}{x^{+}}

\definecolor{codegreen}{rgb}{0,0.3,0.6}
\definecolor{codegray}{rgb}{0.5,0.5,0.5}
\definecolor{codepurple}{rgb}{0.58,0,0.82}
\definecolor{backcolour}{rgb}{0.95,0.95,0.92}

\definecolor{mylavendar}{RGB}{215,131,255}
\definecolor{myblue}{RGB}{0, 150, 255}
\definecolor{mylightblue}{RGB}{118, 214, 255}
\definecolor{mygreen}{RGB}{115, 250, 121}
\definecolor{myred}{RGB}{255, 38, 0}

\lstdefinestyle{mystyle}{
    basicstyle=\tiny,
    commentstyle=\color{codegreen},
    keywordstyle=\color{magenta},
    numberstyle=\tiny\color{codegray},
    stringstyle=\color{codepurple},
    basicstyle=\fontsize{8.5}{9}\selectfont\ttfamily,
    breakatwhitespace=false,         
    breaklines=true,                 
    captionpos=b,                    
    keepspaces=true,                 
    numbers=left,                    
    numbersep=5pt,                  
    showspaces=false,                
    showstringspaces=false,
    frame = single
}

\title{Contrastive Learning with  \\ Hard Negative Samples}

\author{Joshua Robinson, Ching-Yao Chuang, Suvrit Sra, Stefanie Jegelka \\
Massachusetts Institute of Technology\\
Cambridge, MA, USA \\
\texttt{\{joshrob,cychuang,suvrit,stefje\}@mit.edu} \\
}

\newtheorem{thm}{Theorem}

\newtheorem{lemma}[thm]{Lemma}

\newtheorem{prop}[thm]{Proposition}

\newtheorem{principle}[thm]{Principle}

\usepackage{hyperref}
\definecolor{darkblue}{rgb}{0.0,0.0,0.4}
\definecolor{darkred}{rgb}{0.55,0.05,0.0}
\definecolor{darkgreen}{rgb}{0.0,0.29,0.29}
\definecolor{darkpurple}{rgb}{0.47,0.09,0.29}
\definecolor{codepurple}{rgb}{0.58,0,0.82}

\hypersetup{
   colorlinks = true,
   citecolor  = darkblue,
   linkcolor  = darkred,
   filecolor  = darkblue,
 }

\DeclareMathOperator*{\esssup}{ess\,sup}

\newcommand\blfootnote[1]{%
  \begingroup
  \renewcommand\thefootnote{}\footnote{#1}%
  \addtocounter{footnote}{-1}%
  \endgroup
}

\begin{document}

\maketitle

\begin{abstract}
How can you sample good negative examples for contrastive learning? We argue that, as with metric learning, contrastive learning of representations benefits from hard negative samples (i.e., points that are difficult to distinguish from an anchor point). The key challenge toward using hard negatives is that contrastive methods must remain unsupervised, making it infeasible to adopt existing negative sampling strategies that use \emph{true} similarity information. In response, we develop a new family of unsupervised sampling methods for selecting hard negative  samples where the user can control the hardness. A limiting case of this sampling results in a representation that tightly clusters each class, and pushes different classes as far apart as possible. The proposed method improves downstream performance across multiple modalities, requires only few additional lines of code to implement, and introduces no computational overhead. 
\end{abstract}
\vspace{-20pt}
\blfootnote{Code available at: \url{https://github.com/joshr17/HCL}}
\enlargethispage{5pt}
\input{intro.tex}

\input{relatedworks.tex}
\input{background.tex}

\input{method.tex}

\input{analysis.tex}
\input{experiments.tex}

\input{ablations.tex}

\input{discussion.tex}

\pagebreak

{
  \bibliographystyle{iclr2021_conference}
  \bibliography{egbib}
}

\newpage
\appendix
\input{appendix.tex}

\end{document}

%% file: math_commands.tex
\usepackage{amsmath,amsfonts,bm}

\def\eqref#1{equation~\ref{#1}}

\def\1{\bm{1}}

\DeclareMathAlphabet{\mathsfit}{\encodingdefault}{\sfdefault}{m}{sl}
\SetMathAlphabet{\mathsfit}{bold}{\encodingdefault}{\sfdefault}{bx}{n}

\DeclareMathOperator*{\argmax}{arg\,max}
\DeclareMathOperator*{\argmin}{arg\,min}

%% file: intro.tex
\section{Introduction}
\vspace{-5pt}
Owing to their empirical success, contrastive learning methods \citep{chopra2005learning,hadsell2006dimensionality} have become one of the most popular self-supervised approaches for learning representations \citep{oord2018representation,tian2019contrastive,chen2020simple}. In computer vision, unsupervised contrastive learning methods have even outperformed supervised pre-training for object detection and segmentation tasks \citep{misra2020self,he2020momentum}. 

Contrastive learning relies on two key ingredients: notions of similar (positive) $(x,\xp)$ and dissimilar (negative) $(x,\xm)$ pairs of data points. The training objective, typically \emph{noise-contrastive estimation} \citep{gutmann2010noise}, guides the learned representation $f$ to map positive pairs to nearby locations, and negative pairs farther apart; other objectives have also been considered~\citep{chen2020simple}. 
The success of the associated methods depends on the design of informative of the positive and negative pairs, which cannot exploit \emph{true} similarity information since there is no supervision.

Much research effort has addressed sampling strategies for positive pairs, and has been a key driver of recent progress in multi-view and contrastive learning \citep{blum1998combining,xu2013survey,bachman2019learning,chen2020simple,tian2020makes}. 
For image data, positive sampling strategies often apply transformations that preserve semantic content, e.g., jittering, random cropping, separating color channels, etc.\ \citep{chen2020simple,chen2020improved,tian2019contrastive}. 
Such transformations have also been effective in learning control policies from raw pixel data~\citep{srinivas2020curl}. Positive sampling techniques have also been proposed for sentence, audio, and video data \citep{logeswaran2018efficient, oord2018representation, purushwalkam2020demystifying,sermanet2018time}.

Surprisingly, the choice of negative pairs has drawn much less attention in contrastive learning. Often, given an ``anchor'' point $x$, a ``negative'' $x^-$ is simply sampled uniformly from the training data, independent of how informative it may be for the learned representation. In supervised and metric learning settings, ``hard'' (true negative) examples can help guide a learning method to correct its mistakes more quickly \citep{schroff2015facenet,oh2016deep}. For representation learning, informative negative examples are intuitively those pairs that are mapped nearby but should be far apart. This idea is successfully applied in metric learning, where true pairs of dissimilar points are available, as opposed to unsupervised contrastive learning.

With this motivation, we address the challenge of selecting informative negatives for contrastive representation learning. In response, we propose a solution that builds a tunable sampling distribution that prefers negative pairs whose representations are currently very similar. This solution faces two challenges: (1) we do not have access to any true similarity or dissimilarity information; (2) we need an efficient sampling strategy for this tunable distribution. We overcome (1) by building on ideas from positive-unlabeled learning  \citep{elkan2008learning,du2014analysis}, and (2) by designing an efficient, easy to implement importance sampling technique that incurs no computational overhead.

Our theoretical analysis shows that, as a function of the tuning parameter, the optimal representations for our new method place similar inputs in tight clusters, whilst spacing the clusters as far apart as possible. Empirically, our hard negative sampling strategy improves the downstream task performance for image, graph and text data, supporting that indeed, our negative examples are more informative.

\textbf{Contributions.} In summary, we make the following contributions:
\vspace{-8pt}
\begin{enumerate}\setlength{\itemsep}{0pt}
\item We propose a simple distribution over hard negative pairs for contrastive representation learning, and derive a practical importance sampling strategy with zero computational overhead that takes into account the lack of true dissimilarity information;
\item We theoretically analyze the hard negatives objective and optimal representations, showing that they capture desirable generalization properties;
\item We empirically observe that the proposed sampling method improves the downstream task performance on image, graph and text data.
\end{enumerate}
Before moving onto the problem formulation and our results, we summarize related work below.

%% file: relatedworks.tex
\vspace*{-5pt}
\subsection{Related Work}
\vspace*{-5pt}

\textbf{Contrastive Representation Learning.}
Various frameworks for contrastive learning of visual representations have been proposed, including SimCLR  \citep{chen2020simple,chen2020big}, which uses augmented views of other items in a minibatch as negative samples, and  MoCo \citep{he2020momentum,chen2020improved}, which uses a momentum updated memory bank of old negative representations to enable the use of very large batches of negative samples. 
Most contrastive methods are unsupervised, however there exist some that use label information \citep{sylvain2019locality,khosla2020supervised}.
Many works study the role of positive pairs, and, e.g., propose to apply large perturbations for images \cite{chen2020simple,chen2020improved}, or argue to minimize the mutual information within positive pairs, apart from relevant information for the ultimate prediction task \citep{tian2020makes}. Beyond visual data, contrastive methods have been developed for sentence embeddings \citep{logeswaran2018efficient}, sequential data \citep{oord2018representation,henaff2019data}, graph \citep{sun2019infograph,hassani2020contrastive,li2019graph} and node representation learning \citep{velickovic2019deep}, and learning representations from raw images for off-policy control \citep{srinivas2020curl}. The role of negative pairs hase been much less studied. \citet{chuang2020debiased} propose a method for ``debiasing'', i.e., correcting for the fact that not all negative pairs may be true negatives.  It does so by taking the viewpoint of Positive-Unlabeled learning, and exploits a decomposition of the true negative distribution. \cite{kalantidis2020hard} consider applying Mixup \citep{zhang2017mixup} to generate hard negatives in latent space, and \cite{jin2018unsupervised} exploit the specific temporal structure of video to generate negatives for object detection.


\textbf{Negative Mining in Deep Metric Learning.}
As opposed to the contrastive representation learning literature,  selection strategies for negative samples have been thoroughly studied in (deep) metric learning \citep{schroff2015facenet,oh2016deep,harwood2017smart,wu2017sampling,ge2018deep,suh2019stochastic}.  Most of these works observe that it is helpful to use negative samples that are difficult for the current embedding to discriminate. \cite{schroff2015facenet} qualify this, observing that some examples are simply too hard, and propose selecting ``semi-hard'' negative samples. The well known importance of negative samples in metric learning, where (partial) true dissimilarity information is available, raises the question of negative samples in contrastive learning, the subject of this paper.


%% file: background.tex
\vspace*{-5pt}
\section{Contrastive Learning Setup}
\label{sec_crl}
\vspace*{-5pt}

 \begin{figure}[t]
  \centering
\includegraphics[width=\textwidth]{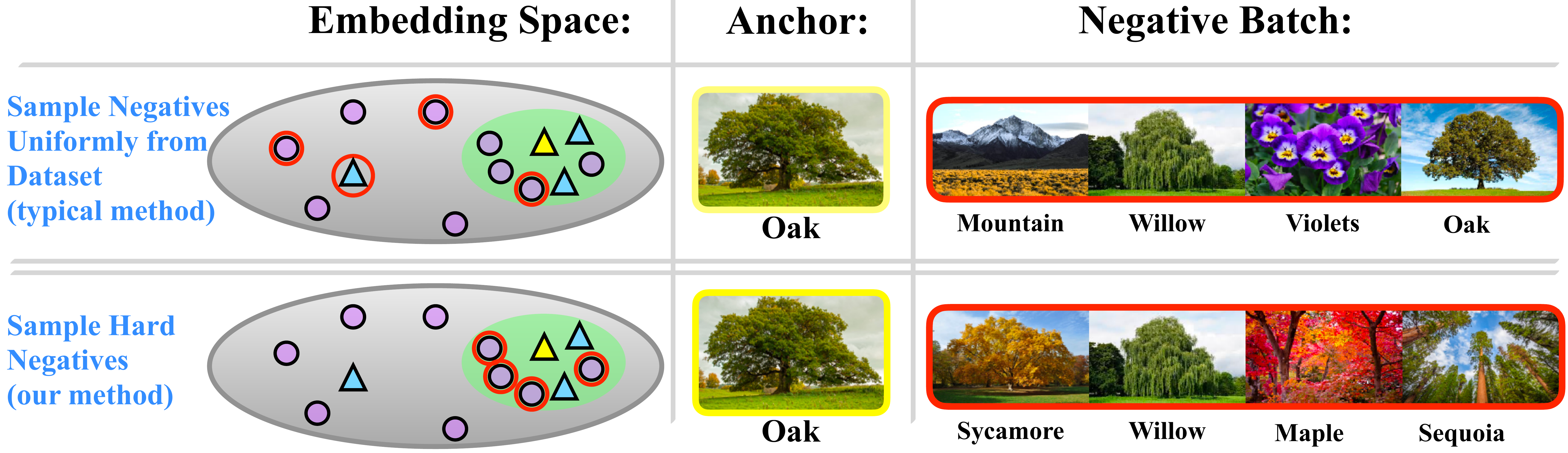} 
\caption{Schematic illustration of negative sampling methods for the example of classifying species of tree. Top row: uniformly samples negative examples (red rings); mostly focuses on very different data points from the anchor (yellow triangle), and may even sample examples from the same class (triangles, vs.\ circles). Bottom row: Hard negative sampling prefers examples that are (incorrectly) close to the anchor.}

\end{figure}

We begin with the setup and the idea of contrastive representation learning. 
 We wish to learn an embedding $f:\mathcal X \rightarrow \mathbb{S}^{d-1}/t$ that maps an observation $x$ to a point on a hypersphere $\mathbb{S}^{d-1}/t$ in $\mathbb{R}^d$ of radius $1/t$, where $t$ is the ``temperature'' scaling hyperparameter. 

Following the setup of \cite{arora2019theoretical}, we assume an underlying set of discrete latent classes $\mathcal{C}$ that represent semantic content, so that similar pairs $(x, x^+)$ have the same latent class. Denoting the distribution over latent classes by $\rho(c)$ for $c \in \mathcal{C}$, we define the joint distribution $p_{x, c}(x,c) = p(x|c)\rho(c)$ whose marginal $p(x)$ we refer to simply as $p$, and assume $\text{supp}(p)=\mathcal X$. For simplicity, we assume $\rho(c) = \tau^+$ is uniform, and let $\tau^- = 1 - \tau^+$ be the probability of another class. Since the class-prior $\tau^+$ is unknown in practice, it must either be treated as a hyperparameter, or estimated \citep{christoffel2016class,jain2016estimating}.

Let $h: \mathcal{X} \rightarrow \mathcal{C}$ be the true underlying hypothesis that assigns class labels to inputs. We write $x \sim x'$ to denote the label equivalence relation $h(x)=h(x')$. We denote by $p_x^+(x^\prime) = p(x^\prime | h(x^\prime) = h(x))$, the distribution over points with same label as $x$, and by $p_x^-(x^\prime) = p(x^\prime | h(x^\prime) \neq h(x))$, the distribution over points with labels different from $x$. We drop the subscript $x$ when the context is clear. Following the usual convention, we overload `$\sim$' and also write $x \sim p$ to denote a point sampled from $p$.

For each data point $x \sim p$, the noise-contrastive estimation (NCE) objective \citep{gutmann2010noise} for learning the representation $f$ uses a \emph{positive} example $x^+$ with the same label as $x$, and \emph{negative} examples $\{x^-_i\}_{i=1}^N$ with (supposedly) different labels, $h(x^-_i) \neq h(x)$, sampled from $q$:

\begin{equation}\label{eqn: practical objective}
 \mathbb{E}_{\substack{x \sim p, \ \ x^+ \sim p_x^+ \\ \{x^-_i\}_{i=1}^N \sim q }} \left [-\log \frac{e^{f(x)^T f(x^+)}}{e^{f(x)^T f(x^+)} + \frac{Q}{N} \sum_{i=1}^N e^{f(x)^T f(x_i^-)} } \right ].
\end{equation}
The weighting parameter $Q$ is introduced for the purpose of analysis. When $N$ is finite we take $Q=N$, yielding the usual form of the contrastive objective. The negative sample distribution $q$ is frequently chosen to be the marginal distribution $p$, or, in practice, an empirical approximation of it  \citep{tian2019contrastive,chen2020simple,chen2020improved,he2020momentum,chen2020improved,oord2018representation,henaff2019data}. In this paper we ask: is there a better way to choose $q$?

%% file: method.tex
\vspace*{-8pt}
\section{Hard Negative Sampling}\label{out method}
\vspace*{-4pt}
In this section we describe our approach for hard negative sampling. We begin by asking \emph{what makes a good negative sample?} 
To answer this question we adopt the following two guiding principles:
\begin{principle}
  $q$ should only sample ``true negatives'' $x_i^-$ whose labels differ from that of the anchor $x$.
\end{principle}

\begin{principle}
  The most useful negative samples are ones that the embedding currently believes to be similar to the anchor. 
\end{principle}

In short, negative samples that have different label from the anchor, but that are embedded nearby are likely to be most useful and provide significant gradient information during training. In metric learning there is access to true negative pairs, automatically fulfilling the first principle. 

In unsupervised contrastive learning there is no supervision, so upholding Principle 1 is impossible to do exactly. In this paper we propose a method that upholds Principle 1 approximately, and simultaneously combines this idea with the key additional conceptual ingredient of ``hardness'' (encapsulated in Principle 2). The level of ``hardness'' in our method can be smoothly adjusted, allowing the user to select the hardness that best trades-off between an improved learning signal from hard negatives, and the harm due to the correction of  false negatives being only approximate. This important since the hardest points are those closest to the anchor, and are expected to have a high propensity to have the same label. Therefore the damage from the approximation  not removing all false negatives becomes larger for harder samples, creating the trade-off.  As a special case our our method, when the hardness level  is tuned fully down, we obtain the method proposed in ~\citep{chuang2020debiased} that only upholds Principle 1 (approximately) but not Principle 2.
 Finally, beyond Principles 1 and 2, we wish to design an efficient sampling method that does not add additional computational overhead during training.

\vspace*{-4pt} 
\subsection{Proposed Hard Sampling Method}
\vspace*{-4pt}
Our first goal is to design a distribution $q$ on $\mathcal X$ that is allowed to depend on the embedding $f$ and  the anchor $x$. From $q$ we sample a batch of negatives $\{x^-_i\}_{i=1}^N$ according to the principles noted above. 
We propose sampling negatives from the distribution $q_\beta^-$ defined as

\[ q_\beta^-(x^-) := q_\beta(x^- | h(x) \neq h(x^-)),\quad\text{where}\quad q_\beta(x^-) \propto e^{\beta f(x)^\top f(x^-)}\cdot  p(x^-), \] 

for $\beta \geq 0$. Note that $q_\beta^-$ and $q_\beta$ both depend on $x$, but we suppress the dependance from the notation. The exponential term in $q_\beta$ is an unnormalized von Mises–Fisher distribution with mean direction $f(x)$ and ``concentration parameter'' $\beta$~\citep{mardia00}.
There are two key components to $q_\beta^-$, corresponding to each principle: 1) conditioning on the event  $\{ h(x) \neq h(x^-)\}$ which guarantees that $(x,x^-)$ correspond to different latent classes (Principle $1$); 2) the concentration parameter $\beta$ term controls the degree by which $q_\beta$ up-weights points $x^-$ that have large inner product (similarity) to the anchor $x$ (Principle $2$). Since $f$ lies on the surface of a hypersphere of radius $1/t$, we have $\| f(x) - f(x') \|^2 = 2/t^2 - 2f(x)^\top f(x')$ so preferring points with large inner product is equivalent to preferring points with small squared Euclidean distance.

Although we have  designed $q_\beta^-$ to have all of the desired components, it is not clear how to sample efficiently from it. To work towards a practical method, note  that we can rewrite this distribution by adopting a PU-learning viewpoint \citep{elkan2008learning,du2014analysis,chuang2020debiased}. That is, by conditioning on the event $\{h(x) = h(x^-) \}$ we can split $q_\beta(x^-)$ as
\begin{equation}\label{eqn: PU equation}  q_\beta(x^-) =   \tau^-  q_\beta^-(x^-) + \tau^+ q^+_\beta(x^-),
\end{equation}
where $q^+_\beta(x^-)=  q_\beta(x^- | h(x) = h(x^-)) \propto e^{\beta f(x)^\top f(x^-)}\cdot  p^+(x^-) $. Rearranging~\eqref{eqn: PU equation} yields a formula $q_\beta^-(x^-) = \bigl ( q_\beta(x^-) - \tau^+ q^+_\beta(x^-) \bigr) / \tau^-$  for the negative sampling distribution $q_\beta^-$ in terms of two distributions that are tractable since we have samples from $p$ and can approximate samples from $p^+$ using a set of semantics-preserving transformations, as is typical in contrastive learning methods.

It is possible to generate samples from $q_\beta$ and (approximately from) $q_\beta^+$ using rejection sampling. However, rejection sampling involves an algorithmic complication since the procedure for sampling batches must be modified. To avoid this, we instead take an importance sampling approach. To obtain this, first note that fixing the number $Q$ and taking the limit $N \rightarrow \infty$ in the objective (\ref{eqn: practical objective}) yields,
\begin{equation}
\mathcal L(f,q) =  \mathbb{E}_{\substack{x \sim p \\ x^+ \sim p_x^+ }} \left [-\log \frac{e^{f(x)^T f(x^+)}}{e^{f(x)^T f(x^+)} + Q \mathbb{E}_{x^- \sim q} [e^{f(x)^T f(x^-)}] } \right ].
\end{equation}

The original objective (\ref{eqn: practical objective}) can be viewed as a finite negative sample approximation to $\mathcal L(f,q)$ (note implicitly $\mathcal L(f,q)$ depends on $Q$) . Inserting $q=q_\beta^-$  and using the rearrangement of equation  (\ref{eqn: PU equation}) we obtain the following hardness-biased objective:
\begin{equation}
 \mathbb{E}_{\substack{x \sim p \\ x^+ \sim p_x^+ }} \left [-\log \frac{e^{f(x)^T f(x^+)}}{e^{f(x)^T f(x^+)} + \frac{Q}{\tau^-}(\mathbb{E}_{x^- \sim q_\beta} [e^{f(x)^T f(x^-)}] - \tau^+ \mathbb{E}_{v \sim q_\beta^+} [e^{f(x)^T f(v) }])} \right ].
 \label{eqn: main objective}
\end{equation}

This  objective suggests that we need only to approximate \emph{expectations}  $\mathbb{E}_{x^- \sim q_\beta} [e^{f(x)^T f(x^-)}]$ and $\mathbb{E}_{v \sim q_\beta^+} [e^{f(x)^T f(v) }]$ over $q_\beta$ and $q_\beta^+$ (rather than explicily sampling). This can be achieved using classical Monte-Carlo importance sampling techniques using samples from $p$ and $p^+$ as follows: 
\begin{align*}
\mathbb{E}_{x^- \sim q_\beta} [e^{f(x)^T f(x^-)}] &= \mathbb{E}_{x^- \sim p} [e^{f(x)^T f(x^-)} q_\beta/p] =  \mathbb{E}_{x^- \sim p} [e^{(\beta+1) f(x)^T f(x^-)} /  Z_\beta],  \\
\mathbb{E}_{v \sim q_\beta^+} [e^{f(x)^T f(v) }] &= \mathbb{E}_{v \sim p^+} [e^{f(x)^T f(v)} q_\beta^+/p^+] =  \mathbb{E}_{v \sim p^+} [e^{(\beta+1) f(x)^T f(v)} / Z_\beta^+],
\end{align*}
where $Z_\beta, Z_\beta^+$ are the partition functions of $q_\beta$ and $q_\beta^+$ respectively. The right hand terms readily admit empirical approximations by replacing $p$ and $p^+$ with  $\hat{p}(x)=\frac{1}{N} \sum_{i=1}^N  \delta_{x_i^-} (x) $ and $\hat{p}^+(x) = \frac{1}{M} \sum_{i=1}^M  \delta_{ x_i^+} (x)$ respectively ($\delta_w$ denotes the Dirac delta function centered at $w$). The only unknowns left are the partition functions,  $Z_\beta=\mathbb{E}_{x^- \sim p} [e^{\beta f(x)^T f(x^-)}]$ and $ Z_\beta^+=\mathbb{E}_{x^+ \sim p^+} [e^{\beta f(x)^T f(x^+)}]$ which themselves are expectations over $p$ and $p^+$ and therefore admit empirical estimates,
\begin{align*}
\widehat{Z}_\beta = \frac{1}{N}  \sum_{i=1}^N e^{\beta f(x)^\top f(x_i^-)} ,  
\qquad
 \widehat{Z}_\beta^+ = \frac{1}{M}  \sum_{i=1}^M e^{ \beta f(x)^\top f(x_i^+)}.
\end{align*}

It is important to emphasize the simplicity of the implementation of our proposed approach. Since we propose to reweight the objective instead of modifying the sampling procedure, only two extra lines of code are needed to implement our approach, with no additional computational overhead. PyTorch-style pseudocode for the objective is given in  Fig. \ref{fig_objective_code} in Appendix \ref{ap: experiments}.

%% file: analysis.tex
\section{Analysis of Hard Negative Sampling}
\vspace*{-5pt}
\subsection{Hard Sampling Interpolates Between Marginal and Worst-Case Negatives}
Intuitively, the concentration parameter $\beta$  in our proposed negative sample distribution $q^-_\beta$ controls the level of ``hardness'' of the negative samples. As discussed earlier, the debiasing method of~\cite{chuang2020debiased} can be recovered as a special case: taking $\beta=0$ to obtain the distribution $q^-_0$. This case amounts to correcting for the fact that some samples in a negative batch sampled from $p$ will have the same label as the anchor. But what interpretation does large $\beta$ admit? Specifically, what does the distribution $q^-_\beta$ converge to in the limit $\beta \rightarrow \infty$, if anything? We show that in the limit $q^-_\beta$ approximates an inner solution to the following zero-sum two player game.
\begin{equation}
 \inf_f \sup_{q \in \Pi} \bigg \{ \mathcal{L}(f,q) = \mathbb{E}_{\substack{x \sim p \\ x^+ \sim p_x^+ }} \bigg [-\log \frac{e^{f(x)^T f(x^+)}}{e^{f(x)^T f(x^+)} +Q \mathbb{E}_{x^- \sim q} [e^{f(x)^T f(x^-)}] } \Big ] \bigg \}.
\end{equation}

where $\Pi = \{ q=q(\cdot ; x, f) : \text{supp}  \left (q(\cdot ; x, f) \right ) \subseteq \{x' \in \mathcal X : x' \nsim x\}, \forall x \in \mathcal X\}$ is the set of distributions with support that is disjoint from points with the same class as $x$ (without loss of generality we assume $ \{x' \in \mathcal X : x' \nsim x\}$ is non-empty). Since $q=q(\cdot ; x, f)$ depends on $x$ and $f$ it can be thought of as a family of distributions. The formal statement is as follows. 
\begin{prop}\label{prop: beta convergence}
Let $\mathcal{L}^*(f) = \sup_{q \in \Pi} \mathcal{L} ( f, q)$. Then for any $t>0$ and  $f : \mathcal X \rightarrow \mathbb{S}^{d-1}/t$ we observe the convergence $ \mathcal{L}(f,q^-_\beta)   \longrightarrow \mathcal{L}^*(f) $ as $\beta \rightarrow \infty$.
\end{prop}

\vspace{-10pt}
\begin{proof}
See Appendix \ref{appendix: first proof}.
\end{proof}

\vspace{-8pt}
To develop a better intuitive understanding of the worst case negative distribution objective $  \mathcal{L}^*(f) = \sup_{q \in \Pi} \mathcal{L} ( f, q)$,  we note that the  supremum can be characterized analytically. Indeed,

\vspace{-8pt}
\begin{align*}
 \sup_{q \in \Pi} \mathcal{L}(f,q)  &= -  \mathbb{E}_{\substack{x \sim p \\ x^+ \sim p_x^+ }} f(x)^T f(x^+) + \sup_{q \in \Pi}   \mathbb{E}_{\substack{x \sim p \\ x^+ \sim p_x^+ }} \log \left \{ e^{f(x)^T f(x^+)} +Q \mathbb{E}_{x^- \sim q} [e^{f(x)^T f(x^-)}]  \right \} \\
&= - \mathbb{E}_{\substack{x \sim p \\ x^+ \sim p_x^+ }} f(x)^T f(x^+) +    \mathbb{E}_{\substack{x \sim p \\ x^+ \sim p_x^+ }} \log  \Big \{ e^{f(x)^T f(x^+)} + Q \cdot \sup_{q \in \Pi} \mathbb{E}_{x^- \sim q} [e^{f(x)^T f(x^-)}]   \Big \}.
\end{align*}

The supremum over $q$ can be pushed inside the expectation since $q$ is a family of distribution indexed by $x$, reducing the problem to maximizing $\mathbb{E}_{x^- \sim q} [e^{f(x)^T f(x^-)}]$, which is solved by any $q^*$ whose support is a subset of $\arg \sup_{x^- : x^- \nsim x} e^{f(x)^T f(x^-)}$ if the supremum is attained. However, computing such points involves maximizing a neural network. Instead of taking this challenging route, using $q^-_\beta$ defines a lower bound by placing higher probability on $x^-$ for which $f(x)^T f(x^-)$ is large. This lower bound becomes tight as $\beta \rightarrow \infty$ (Proposition \ref{prop: beta convergence}).

\subsection{Optimal Embeddings on the Hypersphere for Worst-Case Negative Samples}

What desirable properties does an optimal contrastive embedding (global minimizer of $\mathcal L$) possess that make the representation generalizable? To study this question, we first analyze the distribution of an optimal embedding $f^*$ on the hypersphere when negatives are sampled from the adversarial worst-case distribution. We consider a different limiting viewpoint of objective~(\ref{eqn: practical objective}) as the number of negative samples $N\rightarrow \infty$. Following the formulation of \cite{wang2018understanding} we take $Q=N$ in (\ref{eqn: practical objective}), and subtract $\log N$. This changes neither the set of minimizers, nor the geometry of the loss surface. Taking the number of negative samples $N \rightarrow \infty$ yields the limiting objective, 
\begin{equation}
\mathcal L_\infty (f,q) =  \mathbb{E}_{\substack{x \sim p \\ x^+ \sim p_x^+ }} \bigg [-\log \frac{e^{f(x)^T f(x^+)}}{ \mathbb{E}_{x^- \sim q} [e^{f(x)^T f(x^-)}] } \bigg ] .
\end{equation}
\begin{thm}\label{thm: back packing}
Suppose the downstream task is classification (i.e. $\mathcal C$ is finite), and let $ \mathcal{L}_\infty^* ( f) = \sup_{q \in \Pi} \mathcal{L}_\infty ( f, q)$ . The infimum $ \inf_{f: \; \text{measurable} } \mathcal{L}_\infty^* ( f)$  is attained, and any $f^*$ achieving the global minimum is such that $f^*(x) = f^*(x^+)$ almost surely. Furthermore, letting $\bold v_c = f^*(x)$ for any $x$ such that $h(x)=c$ (so $\bold v_c$ is well defined up to a set of $x$ of measure zero), $f^*$ is characterized as being any solution to the following ball-packing problem,

\vspace{-10pt}
\begin{equation} \max_{ \{ \bold v_c \in \mathbb{S}^{d-1}/t \}_{c \in \mathcal C}} \sum_{c \in \mathcal C} \rho(c) \cdot \min_{c' \neq c} \| \bold v_c - \bold v_{c'} \| ^2 . 
\label{eqn: repulsion eqn}
\end{equation}
\end{thm}

\vspace{-10pt}
\begin{proof}
See Appendix \ref{appendix: second proof}.
\end{proof}
\vspace{-15pt}

\paragraph{Interpretation:} The first component of the result is that $f^*(x) = f^*(x^+)$ almost surely for an optimal $f^*$. That is, an optimal embedding $f^*$ must be invariant across pairs of similar inputs $x,x^+$. The second component is characterizing solutions via the classical geometrical Ball-Packing Problem of \cite{tammes1930origin} (Eq. \ref{eqn: repulsion eqn}) that has only been solved exactly for uniform $\rho$, for specific of  $\abs{\mathcal C}$ and typically for $\mathbb{S}^{2}$ \citep{schutte1951kugel,musin2015tammes,tammes1930origin}. When the distribution $\rho$ over classes is uniform this problem is solved by a set of  $\abs{\mathcal C}$ points on the hypersphere such that the average squared-$\ell_2$ distance from a point to the nearest other point is as large as possible. In other words, suppose we wish to place $\abs{\mathcal C}$  number of balls\footnote{For a manifold $\mathcal M \subseteq \mathbb{R}^d $,  we say $C \subset \mathcal M$ is a ball if it is connected, and there exists a Euclidean ball $\mathcal B = \{ x \in \mathbb{R}^d : \| x \|_2 \leq R\}$  for which $C = \mathcal M \cap \mathcal B$.} on $ \mathbb{S}^{d-1}$ so that they do not intersect. Then solutions to Tammes' Problem (\ref{eqn: repulsion eqn}) expresses (twice) the largest possible average squared radius that the balls can have.  So, we have a ball-packing problem where instead of trying to pack as many balls as possible of a fixed size, we aim to pack a fixed number of balls (one for each class) to have as big radii as possible. Non-uniform $\rho$ adds importance weights to each fixed ball. In summary, solutions of the  problem $\min_f \mathcal{L}_\infty^*( f)$ are a maximum margin clustering.

This understanding of global minimizers of  $\mathcal{L}_\infty^* ( f) = \sup_{q \in \Pi} \mathcal{L}_\infty ( f, q)$ can further developed into a better understanding of generalization on downstream tasks. The next result shows that representations that achieve small excess risk on the objective $\mathcal L_\infty ^*$ still separate clusters well in the sense that a simple 1-nearest neighbor classifier achieves low classification error. 

\begin{thm}
Suppose $\rho$ is uniform on $\mathcal C$ and $f$ is such that $\mathcal L_\infty ^*(f) - \inf_{\bar{f} \; \text{measurable}} \mathcal L_\infty ^*(\bar{f}) \leq \varepsilon$ with $\varepsilon \leq 1$. Let $ \{ \bold v^*_c \in \mathbb{S}^{d-1}/t \}_{c \in \mathcal C}$ be a solution to Problem \ref{eqn: repulsion eqn}, and define the constant $\xi = \min_{c,c^- : c \neq c^-} \norm{\bold v^*_c -\bold v^*_{c-} } > 0$. Then there exists a set of vectors $ \{ \bold v_c \in \mathbb{S}^{d-1}/t \}_{c \in \mathcal C}$ such that the 1-nearest neighbor classifier $\hat{h}(x) = \arg\min_{\bar{c} \in \mathcal C} \norm{f(x) - \bold v_{\bar{c}} }$ (ties broken arbitrarily) achieves misclassification risk,
\begin{equation*} 
\mathbb P_{x,c}(\hat{h}(x) \neq c) \leq \frac{8\varepsilon}{(\xi^2 - 2\abs{\mathcal C} (1+1/t)\varepsilon^{1/2})^2}
\end{equation*}
\label{generalization thm}
\end{thm}

\vspace{-10pt}
\begin{proof}
See Appendix \ref{generalization theorem}.
\end{proof}
\vspace{-1pt}

In particular, $\mathbb P(\hat{h}(x) \neq c) = \mathcal O(\varepsilon)$ as $\varepsilon \rightarrow 0$, and in the limit $\varepsilon \rightarrow 0$ we recover the invariance claim of Theorem \ref{thm: back packing} as a special case. The result can be generalized to arbitrary $\rho$ by replacing $\abs{\mathcal C}$ in the bound by $1/ \min_c \rho(c)$. The result also implies that it is possible to build simple classifiers for tasks that involve only a subset of classes from $\mathcal C$, or classes that are a union of classes from $\mathcal C$. The constant $\xi = \min_{c,c^- : c \neq c^-} \norm{\bold v^*_c -\bold v^*_{c-} } > 0$ is a purely geometrical property of spheres, and describes the minimum separation distance between a set of points that solves the Tammes' ball-packing problem.

\vspace{+6pt}

%% file: experiments.tex
\vspace{-5pt}
\section{Empirical Results}
\vspace{-10pt}

Next, we evaluate our hard negative sampling method empirically, and apply it as a modification to state-of-the-art contrastive methods on image, graph, and text data. For all experiments  $\beta$ is treated as a hyper-parameter (see ablations in Fig. \ref{fig: stl10 cifar10 accuracies}  for more understanding of how to pick $\beta$). Values for $M$ and $\tau^+$ must also be determined. We fix $M=1$ for all experiments, since taking $M>1$ would increase the number of inputs for the forward-backward pass. Lemma \ref{lemma: bias/var} in the appendix gives a theoretical justification for the choice of $M=1$. 
Choosing  the class-prior $\tau^+$ can be done in two ways: estimating it from data \citep{christoffel2016class,jain2016estimating}, or treating it as a hyper-parameter. The first option requires the possession of labeled data \emph{before} contrastive training.

\subsection{Image Representations} 
\vspace{-2pt}

We begin by testing the hard sampling method on vision tasks using the STL10, CIFAR100 and CIFAR10 data. We use SimCLR \citep{chen2020simple} as the baseline method, and all models are trained for $400$ epochs. The results in Fig. \ref{fig: stl10 cifar10 accuracies} show consistent improvement over SimCLR ($q=p$) and the particular case of our method with $\beta=0$ proposed in \citep{chuang2020debiased} (called debiasing) on STL10 and CIFAR100. For $N=510$ negative examples per data point we observe absolute improvements of  $3\%$ and $7.3\%$ over SimCLR on CIFAR100 and STL10 respectively, and absolute improvements over the  best debiased baseline of $1.9\%$ and $3.2\%$. On tinyImageNet (Tab. \ref{tab: tinyImageNet}) we observe an absolute improvement of $3.6\%$ over SimCLR, while on CIFAR10 there is a slight improvement for smaller $N$, which disappears at larger $N$. See Appendix \ref{sec: moco} results using MoCo-v2 for large negative batch size, and Appendix \ref{vision experiments} for full setup details.

 \vspace{-5pt}
\begin{wraptable}{r}{5.5cm}
 \vspace{-30pt}
\begin{tabular}{ccc}\\\toprule  
SimCLR & Debiased & Hard ($\beta=1$) \\\midrule
53.4\% & 53.7\% & 57.0\% \\  \bottomrule
\end{tabular}
\caption{Top-1 linear readout on tinyImageNet. Class prior is set to $\tau^+=0.01$. }
 \vspace{-10pt}
 \label{tab: tinyImageNet}
\end{wraptable} 

 \vspace{-3pt}
 \begin{figure}[h]
  \centering
  \includegraphics[width=45mm]{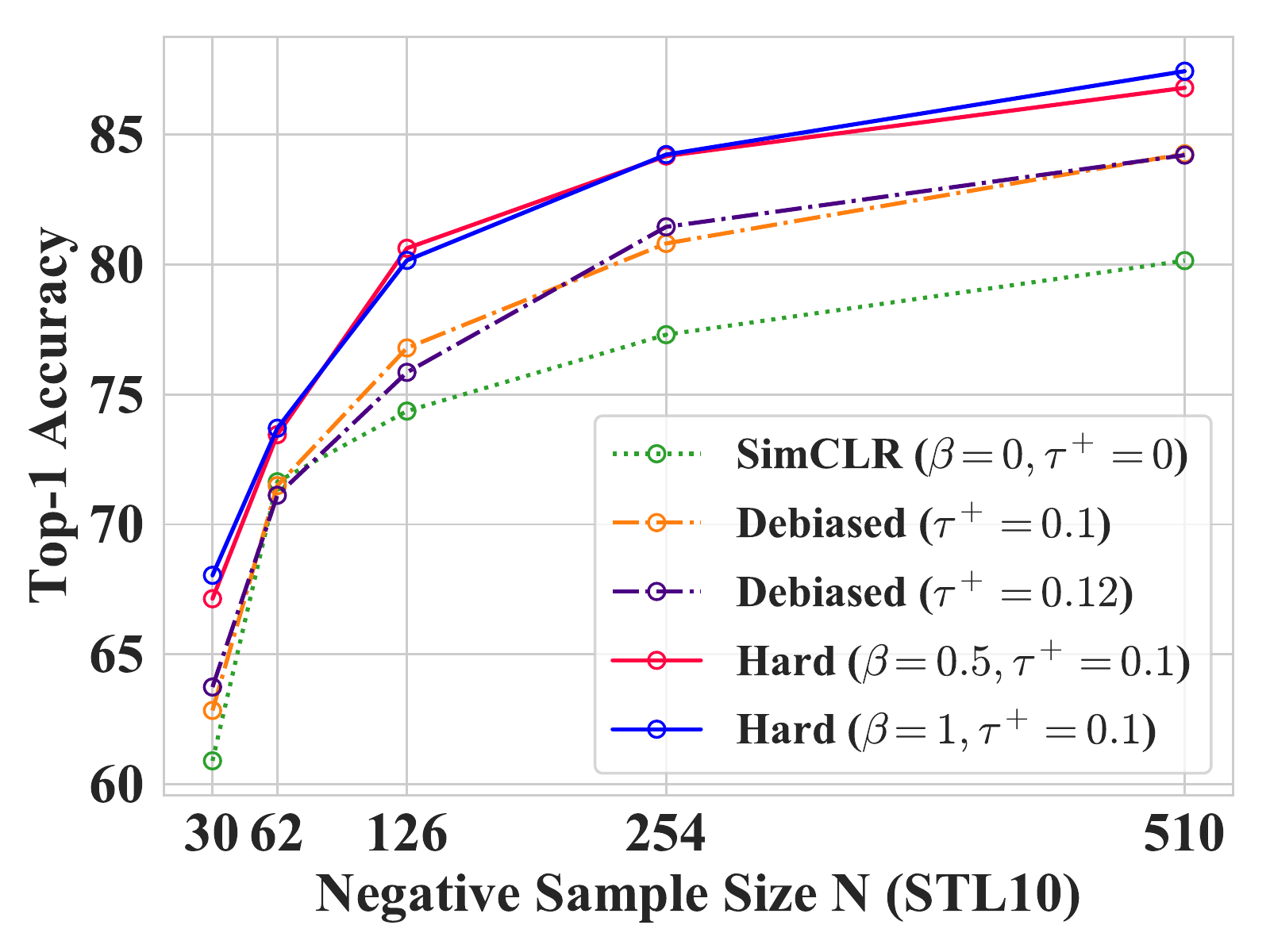}
  \includegraphics[width=45mm]{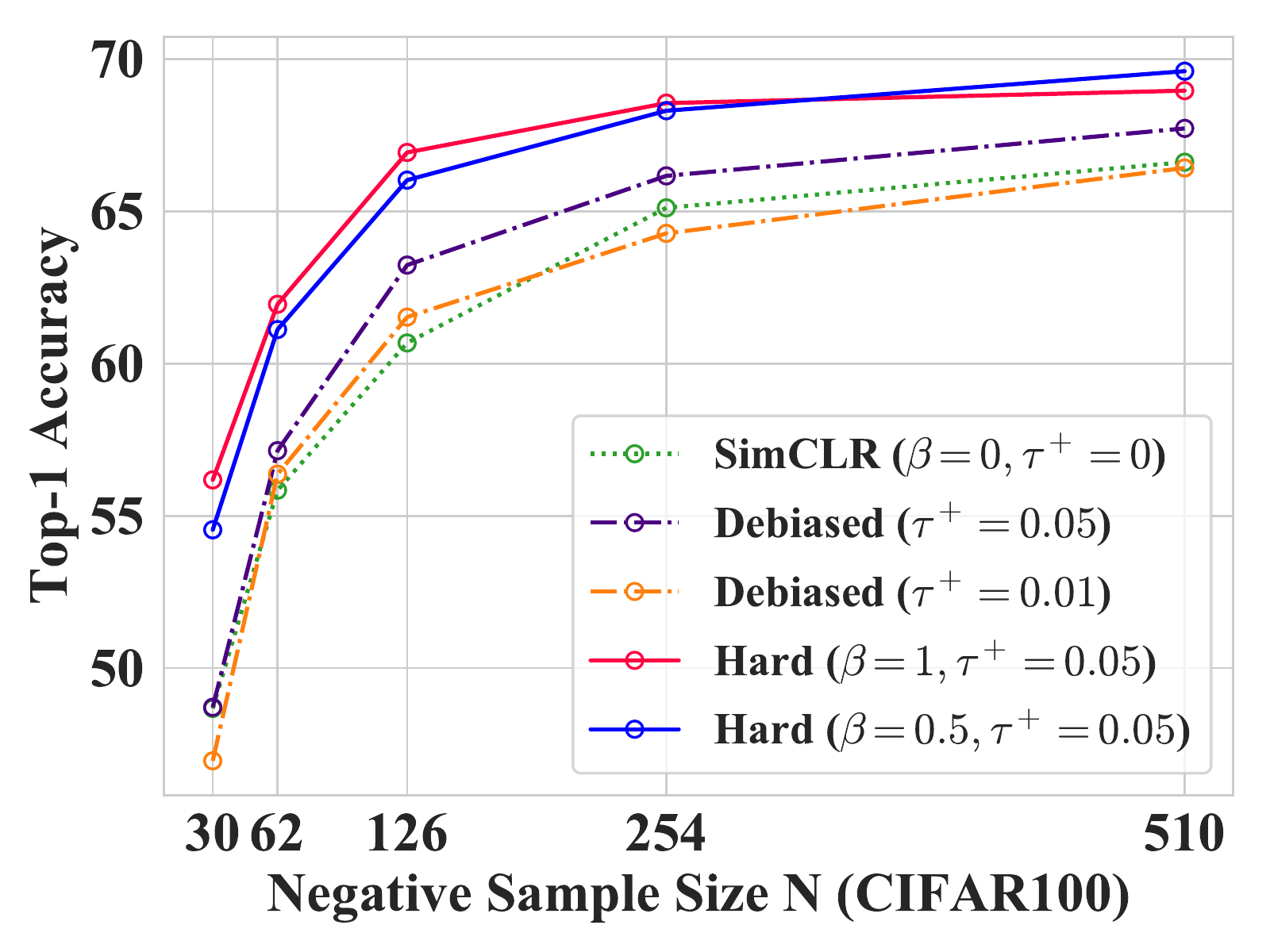}           \includegraphics[width=45mm]{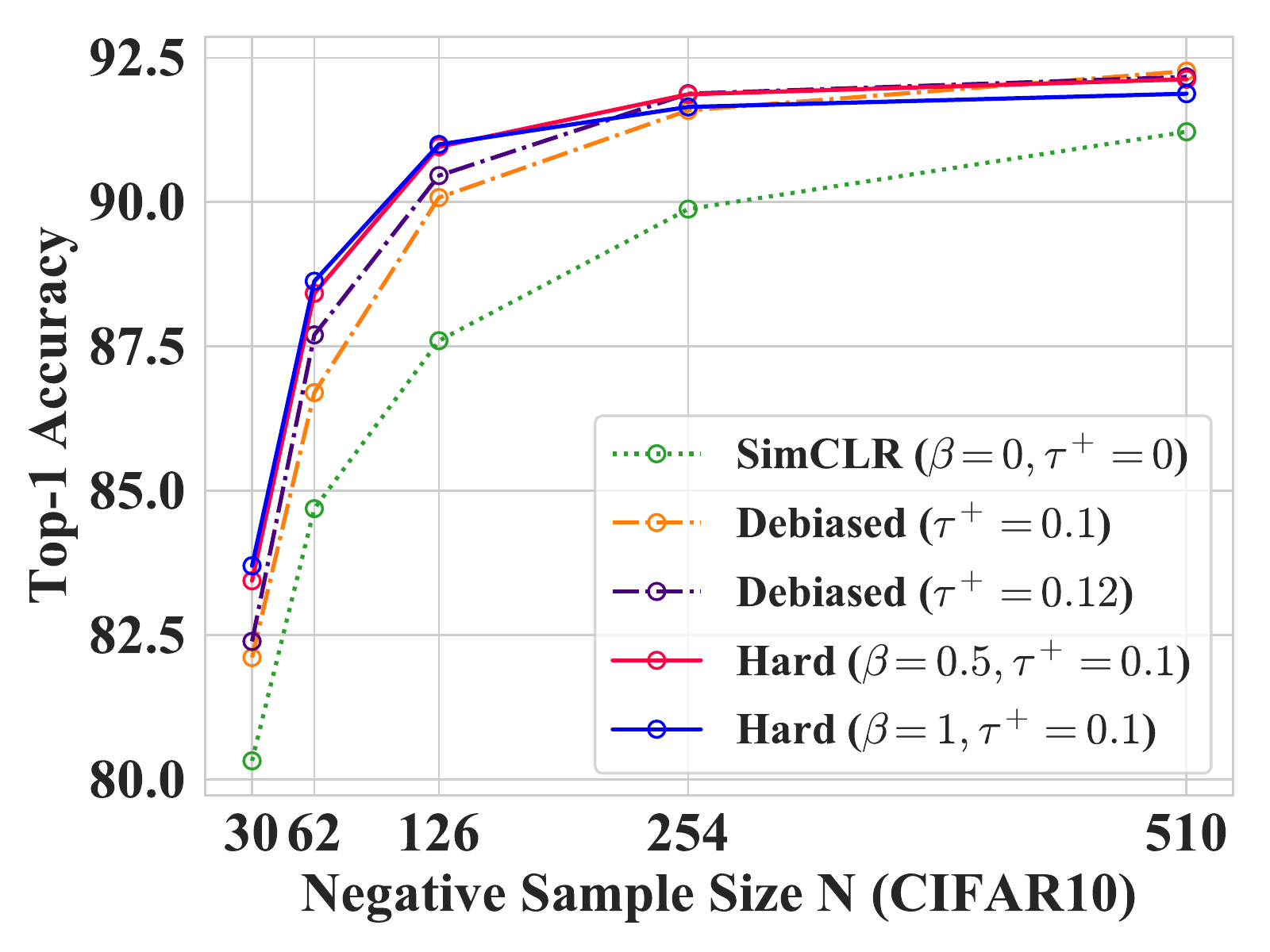}  
 \vspace{-3pt}
    \caption{\textbf{Classification accuracy on downstream tasks.}  Embeddings trained using hard, debiased, and standard ($\beta=0, \tau^+=0$) versions of SimCLR, and evaluated using linear readout accuracy. }%
        \label{fig: stl10 cifar10 accuracies}
 \end{figure}
 \vspace{-10pt}
 
\subsection{Graph Representations } \label{section: graph classification}
\vspace{-6pt}

Second, we consider hard negative sampling in the context of learning graph representations. We use the state-of-the-art InfoGraph method introduced by \cite{sun2019infograph} as the baseline, which is suitable for downstream graph-level classification. 
The objective is of a slightly different form from the NCE loss.
Because of this we use a  generalization of the formulation presented in Section~\ref{out method} (See Appendix~\ref{appendix: graph method} for details). In doing so, we illustrate that it is easy to adapt our hard sampling method to other contrastive frameworks. 

Fig. \ref{graph_table} shows the results of
fine-tuning an SVM \citep{boser1992training,cortes1995support} on the fixed, learned embedding for a range of different values of $\beta$. Hard sampling does as well as InfoGraph in all cases, and better in 6 out of 8 cases. For ENZYMES and REDDIT, hard negative samples improve the accuracy by  $3.2\%$ and $2.4\%$, respectively, for DD and PTC by $1-2\%$, and for IMDB-B and MUTAG by at least $0.5\%$. Usually, multiple different choices of $\beta>0$ were competitive with the InfoGraph baseline: 17 out of the 24 values of $\beta>0$ tried (across all 8 datasets) achieve accuracy as high or better than InfoGraph ($\beta=0$).

 \begin{figure}[t]
  \centering
\includegraphics[width=\textwidth]{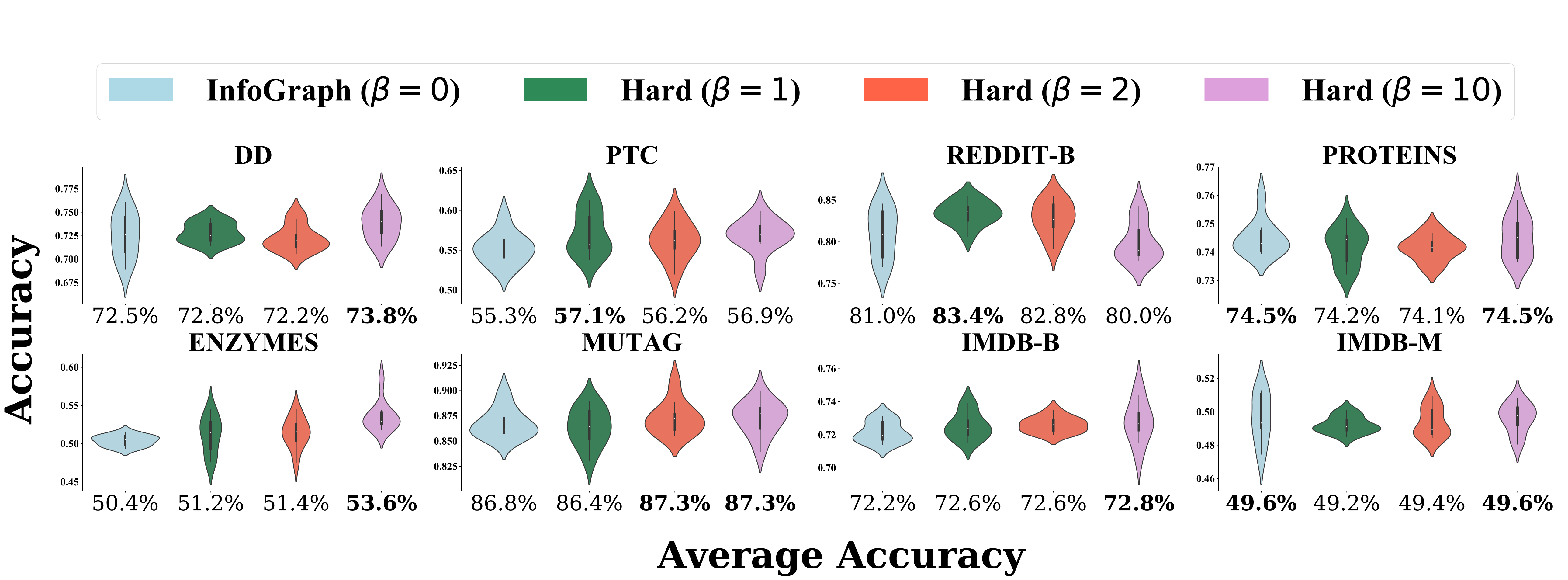}  
\vspace{-10pt}
\caption{\textbf{Classification accuracy on downstream tasks.} We compare graph representations on four classification tasks. Accuracies are obtained by fine-tuning an SVM readout function, and are the average of $10$ runs, each using $10$-fold cross validation. Results in \textbf{bold} indicate best performer.}
\vspace{-10pt}
\label{graph_table}
 \end{figure}

\vspace{-4pt}
\subsection{Sentence Representations}
\vspace{-6pt}

Third, we test hard negative sampling on learning representations of sentences using the \emph{quick-thoughts} (QT) vectors framework introduced by \cite{logeswaran2018efficient}, which uses adjacent sentences (before/after) as positive samples. Embeddings are trained using the unlabeled BookCorpus dataset \citep{kiros2015skip}, and evaluated following the protocol of \cite{logeswaran2018efficient} on six downstream tasks. The results are reported in Table \ref{sent_comp}. Hard sampling outperforms or equals the QT baseline in 5 out of 6 cases, the debiased baseline \citep{chuang2020debiased}  in 4 out of 6, and both in 3 out of 6 cases. Setting $\tau^+ >0$ led to numerical issues in optimization for hard sampling. 
\begin{table*}[!ht]
\small
\begin{center}
\begin{tabularx}{0.80\textwidth}{l| *{7}{c}}
\hline
 \fontsize{9pt}{9pt}\selectfont\textbf{Objective} &  \fontsize{9pt}{9pt}\selectfont\textbf{MR} & \fontsize{9pt}{9pt}\selectfont\textbf{CR} & \fontsize{9pt}{9pt}\selectfont\textbf{SUBJ} & \fontsize{9pt}{9pt}\selectfont\textbf{MPQA} & \fontsize{9pt}{9pt}\selectfont\textbf{TREC} & \multicolumn{2}{c}{\fontsize{9pt}{9pt}\selectfont\textbf{MSRP}}
 \\
& & & & & &  \fontsize{9pt}{9pt}\selectfont\textbf{(Acc)} & \fontsize{9pt}{9pt}\selectfont\textbf{(F1)} \\
\hline
\hline
QT ($\beta=0,  \tau^+ = 0$) & 76.8 & 81.3 & 86.6 & 93.4 & \textbf{89.8} & 73.6 & 81.8
\\
Debiased ($\tau^+ = 0.01$) & 76.2 & 82.9 & 86.9 & \textbf{93.7} & 89.1 & \textbf{74.7} & \textbf{82.7}
\\
\hline
Hard ($\beta=1,\tau^+ = 0$) & 77.1 &	 82.5 & \textbf{87.0} & 92.9 & 89.2	& 73.9 & 82.2
\\
Hard ($ \beta=2, \tau^+ = 0$) & \textbf{77.4} & \textbf{83.6} & 86.8 & 93.4 & 88.7 & 73.5 & 82.0	
\\
\hline
\end{tabularx}
\end{center}
\vspace{-6pt}
\caption{\textbf{Classification accuracy on downstream tasks.} Sentence representations are learned using quick-thoughts (QT) vectors on the BookCorpus dataset and evaluated on six classification tasks.  Evaluation of binary classification tasks (MR, CR, SUBJ, MPQA) uses 10-fold cross validation.
}
\label{sent_comp}
\end{table*}
\vspace{-12pt}

%% file: ablations.tex
 \vspace{-5pt}
\section{A Closer Look at Hard Sampling}
\vspace{-2pt}

\subsection{Are Harder Samples Necessarily Better?}\label{sec: harder is not necessarily better}
\vspace{-4pt}

By setting $\beta$ to large values, one can focus on only the hardest samples in a training batch. But is this desirable? Fig. \ref{fig:stl ablation} (left, middle)  shows that for vision problems, taking larger $\beta$ does not necessarily lead to better representations. In contrast, when one uses true positive pairs during training (green curve, uses label information for positive but not negative pairs), the downstream performance monotonically increases with $\beta$ until convergence (Fig. \ref{fig:stl ablation} , middle). Interestingly, this is achieved without using label information for the negative pairs.
This observation suggests an explanation for why bigger $\beta$ hurts performance in practice. Debiasing (conditioning on the event  $\{ h(x) \neq h(x^-)\}$) using the true $p^+$ corrects for sampling $x^-$ with the same label as $x$. However, since in practice we approximate $p^+$ using a set of data transformations, we can only partially correct. This is harmful for large $\beta$ since this regime strongly prefers $x^-$ for which $f(x^-)$ is close to $f(x)$, many of whom will have the same label as $x$ if not corrected for. We note also that by annealing $\beta$ (gradually decreasing $\beta$ to $0$ throughout training; see Appendix \ref{vision experiments} for details) it is possible to be more robust to the choice of initial $\beta$, with marginal impact on downstream accuracy compared to the best fixed value of $\beta$.

\subsection{Does Avoiding False Negatives Improve Hard Sampling?}
\vspace{-4pt}

Our proposed hard negative sampling method conditions on the event  $\{ h(x) \neq h(x^-)\}$ in order to avoid false negatives (termed ``debiasing'' \citep{chuang2020debiased}). But does this help? To test this, we train four embeddings: hard sampling with and without debiasing, and uniform sampling ($\beta=0$) with and without debiasing. The results in Fig. \ref{fig:stl ablation} (right) show that hard sampling with debiasing obtains the highest linear readout accuracy on STL10, only using hard sampling or only debiasing yields (in this case) similar accuracy. All improve over the SimCLR baseline.

 \begin{figure}[t]
  \centering
 \includegraphics[width=0.32\textwidth]{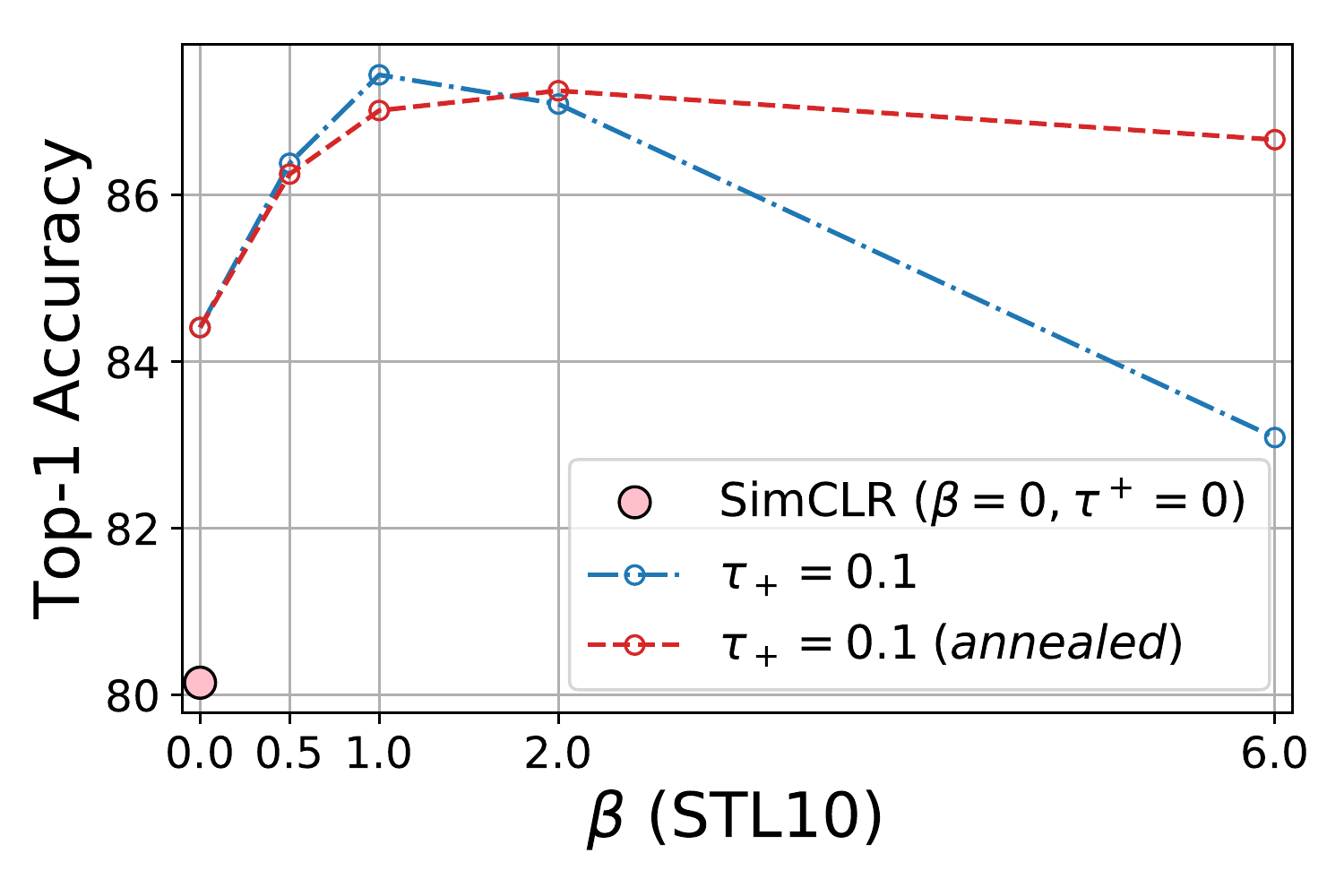}  
 \includegraphics[width=0.32\textwidth]{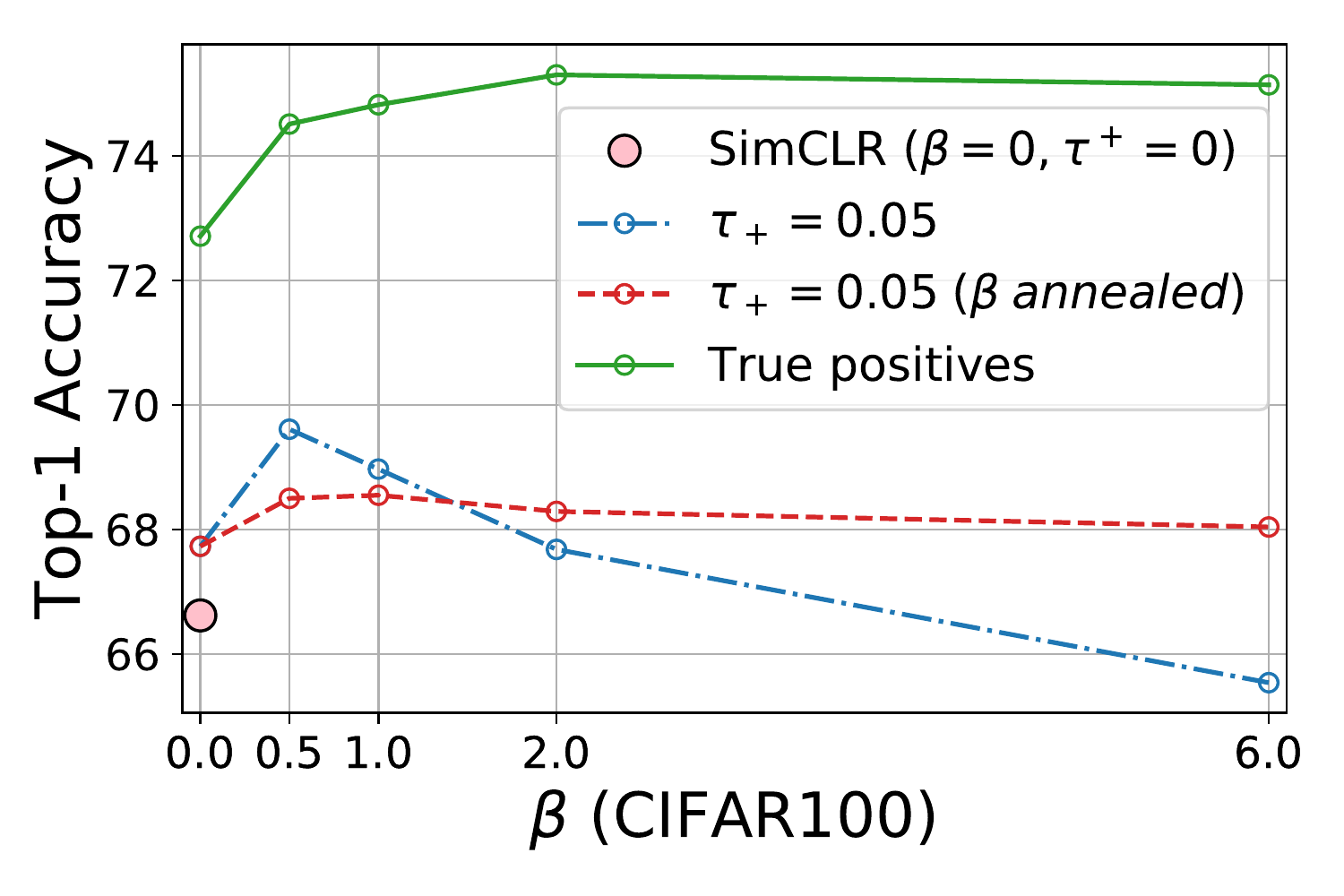}
 \includegraphics[width=0.32\textwidth]{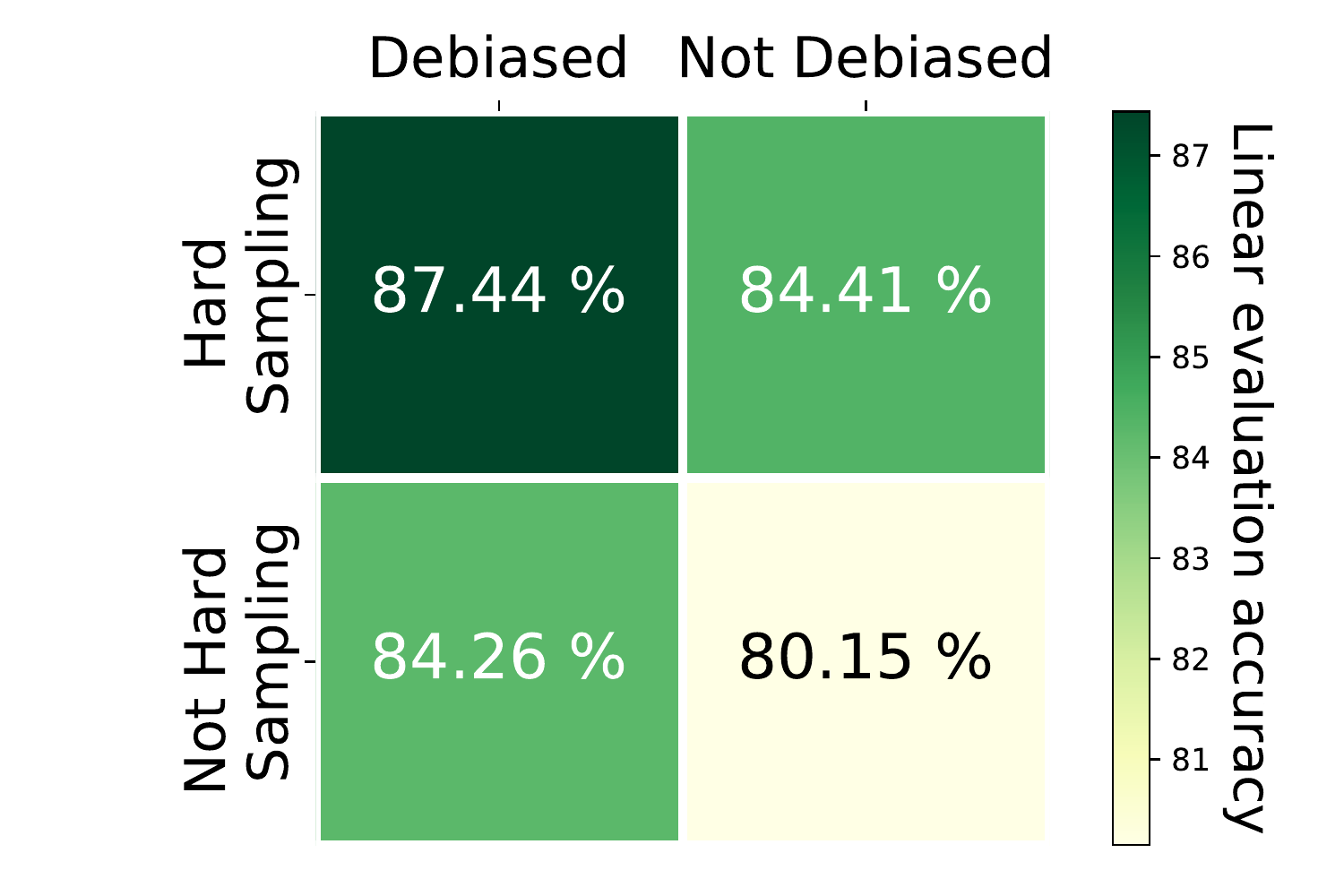}
	\caption{\label{fig:stl ablation}Left: the effect of varying concentration parameter $\beta$ on linear readout accuracy. Middle: linear readout accuracy as concentration parameter $\beta$ varyies, in the case of contrastive learning (fully unsupervised),  using true positive samples (uses label information), and an annealing method that improves robustness to the choice of $\beta$ (see Appendix \ref{vision experiments} for details). Right: STL10 linear readout accuracy  for hard sampling with and without debiasing, and non-hard sampling ($\beta=0$) with and without debiasing. Best results come from using both simultaneously.}
\vspace{-10pt}
 \end{figure}

Fig. \ref{fig:stl histogram ablation}  compares the histograms of cosine similarities of positive and negative pairs for the four learned representations. The representation trained with hard negatives and debiasing assigns much lower similarity score to a pair of negative samples than other methods. On the other hand, the SimCLR baseline assigns higher cosine similarity scores to pairs of positive samples. 
However, to discriminate positive and negative pairs, a key property is the amount of \emph{overlap} of positive and negative histograms. Our hard sampling method achieves less overlap than SimCLR, by better trading off higher dissimilarity of negative pairs with less similarity of positive pairs.  Similar tradeoffs are observed for the debiased objective, and hard sampling without debiasing.

\begin{figure}[t]
    \includegraphics[width=\textwidth]{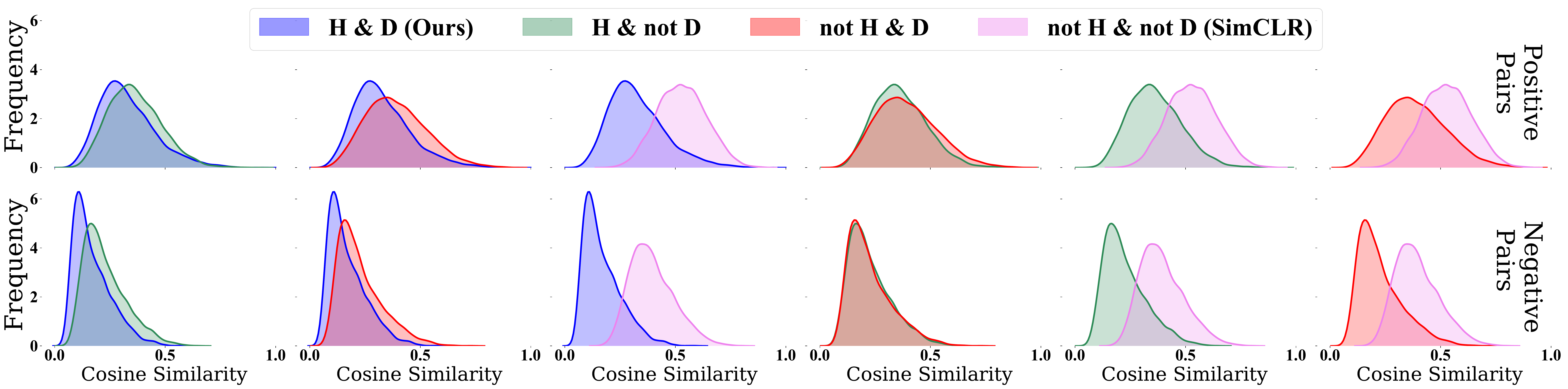}
    \caption{Histograms of cosine similarity of pairs of points with the same label (top) and different labels (bottom) for embeddings trained on STL10 with four different objectives. H$=$Hard Sampling, D$=$Debiasing.  Histograms overlaid pairwise to allow for convenient comparison. }
   \vspace{-8pt}
    \label{fig:stl histogram ablation}
\end{figure}

\subsection{How do Hard Negatives Affect Optimization?}
\vspace{-6pt}

Fig. \ref{fig: optimization} (in Appendix \ref{sec: app ablations} due to space constraints) shows the performance on STL10 and CIFAR100 of SimCLR versus using hard negatives throughout training. We use weighted $k$-nearest neighbors with $k=200$ as the classifier and evaluate each model once every five epochs. Hard sampling with $\beta=1$ leads to much faster training: on STL10 hard sampling takes only 60 epochs to reach the same performance as SimCLR does in 400 epochs. On CIFAR100 hard sampling takes only 125 epochs to reach the same performance as SimCLR does in 400 epochs. We speculate that the speedup is, in part, due to hard negatives providing non-negligible gradient information during training.

%% file: discussion.tex
\vspace{-8pt}
\section{Conclusion}
\vspace{-8pt}

 We argue for the value of hard negatives in unsupervised contrastive representation learning, and introduce a simple hard negative sampling method. Our work connects two major lines of work: contrastive learning, and negative mining in metric learning. Doing so requires overcoming an apparent roadblock: negative mining in metric learning uses pairwise similarity information as a core component, while contrastive learning is unsupervised. Our method enjoys several nice aspects: having desirable theoretical properties, a very simple implementation that requires modifying only a couple of lines of code, not changing anything about the data sampling pipeline, introducing zero extra computational overhead, and handling false negatives in a principled way.

%% file: appendix.tex
\section{Analysis of Hard Sampling}

\subsection{Hard Sampling Interpolates Between Marginal and Worst-Case Negatives}\label{appendix: first proof}

We begin by proving Proposition \ref{prop: beta convergence}. Recall that the proposition stated the following.

\begin{prop}
Let $\mathcal{L}^*(f) = \sup_{q \in \Pi} \mathcal{L} ( f, q)$. Then for any $t>0$ and measurable $f : \mathcal X \rightarrow \mathbb{S}^{d-1}/t$ we observe the convergence $ \mathcal{L}(f,q^-_\beta)   \longrightarrow \mathcal{L}^*(f) $ as $\beta \rightarrow \infty$.
\end{prop}

\begin{proof}
Consider the following essential supremum, 
\[M(x) = \esssup_{x^- \in \mathcal X: x^- \nsim x}f(x)^T f(x^-) =   \sup \{ m >0 : m \geq  f(x)^T f(x^-) \; \text{a.s. for } x^- \sim p^-\}.\]

 The second inequality holds since $\text{supp}(p) = \mathcal X$. We may rewrite 

\begin{align*}
&\mathcal{L}^*(f) =  \mathbb{E}_{\substack{x \sim p \\ x^+ \sim p_x^+ }} \left [-\log \frac{e^{f(x)^T f(x^+)}}{e^{f(x)^T f(x^+)} + Q e^{M(x)} } \right ] ,  \\
&\mathcal{L}(f,q^-_\beta) = \ \mathbb{E}_{\substack{x \sim p \\ x^+ \sim p_x^+ }} \left [-\log \frac{e^{f(x)^T f(x^+)}}{e^{f(x)^T f(x^+)} + Q \mathbb{E}_{x^- \sim q^-_\beta} [e^{f(x)^T f(x^-)}] } \right ].
\end{align*}

The difference between these two terms can be bounded as follows,

\begin{align*}
\abs{\mathcal{L}^*(f)  - \mathcal{L}(f,q^-_\beta)} &\leq   \mathbb{E}_{\substack{x \sim p \\ x^+ \sim p_x^+ }} \left | -\log \frac{e^{f(x)^T f(x^+)}}{e^{f(x)^T f(x^+)} + Q e^{M(x)} }   + \log \frac{e^{f(x)^T f(x^+)}}{e^{f(x)^T f(x^+)} + Q \mathbb{E}_{x^- \sim q^-_\beta} [e^{f(x)^T f(x^-)}] } \right | \\
&=  \mathbb{E}_{\substack{x \sim p \\ x^+ \sim p_x^+ }} \left | \log \left (e^{f(x)^T f(x^+)} + Q \mathbb{E}_{x^- \sim q^-_\beta} [e^{f(x)^T f(x^-)}]  \right ) - \log \left ( e^{f(x)^T f(x^+)} + Q e^{M(x)} \right ) \right | \\
&\leq  \frac{e^{1/t}}{Q+1} \cdot \mathbb{E}_{\substack{x \sim p \\ x^+ \sim p_x^+ }} \left | e^{f(x)^T f(x^+)} + Q \mathbb{E}_{x^- \sim q^-_\beta} [e^{f(x)^T f(x^-)}]   -  e^{f(x)^T f(x^+)} - Q e^{M(x)} \right | \\
&= \frac{e^{1/t}Q}{Q+1} \cdot  \mathbb{E}_{\substack{x \sim p }} \left |  \mathbb{E}_{x^- \sim q^-_\beta} [e^{f(x)^T f(x^-)}]   - e^{M(x)} \right | \\
&\leq e^{1/t} \cdot    \mathbb{E}_{x \sim p} \mathbb{E}_{x^- \sim q^-_\beta} \abs{ e^{M(x)} -  e^{f(x)^T f(x^-)}} 
\end{align*}

where for the second inequality we have used the fact that $f$ lies on the  hypersphere of radius $1/t$ to restrict the domain of the logarithm to values greater than $(Q+1)e^{-1/t}$. Because of this the logarithm is Lipschitz with parameter $e^{1/t}/(Q+1)$. Using again the fact that $f$ lies on the hypersphere we know that $\abs{f(x)^T f(x^-)} \leq 1/t^2$ and hence have the following inequality, 

\begin{align*}
\mathbb{E}_{x \sim p} \mathbb{E}_{q^-_\beta} \abs{ e^{M(x)} -  e^{f(x)^T f(x^-)}} 
&\leq e^{1/t^2} \mathbb{E}_{x \sim p} \mathbb{E}_{q^-_\beta} \abs{ M(x) - f(x)^T f(x^-)} 
\end{align*}

Let us consider the inner expectation $E_\beta(x ) = \mathbb{E}_{q^-_\beta} \abs{ M(x) - f(x)^T f(x^-)}$. Note that since $f$ is bounded, $E_\beta(x)$ is uniformly bounded in $x$. Therefore, in order to show the convergence $\mathcal{L}(f,q^-_\beta)   \rightarrow \mathcal{L}^*(f) $ as $\beta \rightarrow \infty$, it suffices by the dominated convergence theorem to show that $E_\beta(x) \rightarrow 0$ pointwise as $\beta \rightarrow \infty$ for arbitrary fixed $x \in \mathcal X$.

From now on we denote $M=M(x)$ for brevity, and consider a fixed $x \in \mathcal X$. From the definition of $q^-_\beta$ it is clear that $q^-_\beta \ll p^-$. That is, since $q^-_\beta = c \cdot p^-$ for some (non-constant) $c$, it is absolutely continuous with respect to $p^-$. So $M(x) \geq f(x)^T f(x^-)$ almost surely for $x^- \sim q^-_\beta $, and we may therefore drop the absolute value signs from our expectation. Define the following event $\mathcal G_\varepsilon = \{ x^- : f(x)^\top f(x^-) \geq M - \varepsilon \}$ where $\mathcal G$ is refers to a ``good'' event. Define its complement $\mathcal B_\varepsilon = \mathcal G_\varepsilon^c$ where $\mathcal B$ is for ``bad''. For a fixed $x \in \mathcal X$ and $\varepsilon > 0$ consider,

\begin{align*}
E_\beta(x) &=  \mathbb{E}_{x^- \sim q^-_\beta} \abs{ M(x) - f(x)^T f(x^-)} \\
& = \mathbb{P}_{x^- \sim q^-_\beta}(  \mathcal G_\varepsilon) \cdot   \mathbb{E}_{x^- \sim q^-_\beta} \left [ \abs{ M(x) - f(x)^T f(x^-)} |  \mathcal G_\varepsilon \right ]  +   \mathbb{P}_{x^- \sim q^-_\beta}(  \mathcal B_\varepsilon) \cdot \mathbb{E}_{x^- \sim q^-_\beta} \left [ \abs{ M(x) - f(x)^T f(x^-)} |  \mathcal B_\varepsilon \right ] \\
& \leq \mathbb{P}_{x^- \sim q^-_\beta}(  \mathcal G_\varepsilon) \cdot   \varepsilon  + 2  \mathbb{P}_{x^- \sim q^-_\beta}(  \mathcal B_\varepsilon) \\
& \leq  \varepsilon  + 2  \mathbb{P}_{x^- \sim q^-_\beta}(  \mathcal B_\varepsilon) .
\end{align*}

We need to control $ \mathbb{P}_{x^- \sim q^-_\beta}(  \mathcal B_\varepsilon)$. Expanding, 

\vspace{-10pt}
\begin{align*}
\mathbb{P}_{x^- \sim q^-_\beta}(  \mathcal B_\varepsilon) &= \int_{\mathcal X} \mathbf{1} \left  \{  f(x)^T f(x^-) < M(x) -\varepsilon \right \} \frac{e^{\beta  f(x)^T f(x^-)}\cdot p^-(x^-)}{Z_\beta} \text{d}{x^-}
\end{align*}

where $Z_\beta = \int_{ \mathcal X} e^{\beta  f(x)^T f(x^-)} p^-(x^-) \text{d}{x^-}$ is the partition function of $q^-_\beta$. We may bound this expression by,
\vspace{-10pt}

\begin{align*}
\int_{\mathcal X} \mathbf{1} \left  \{  f(x)^T f(x^-) < M -\varepsilon \right \} \frac{e^{\beta  (M -\varepsilon)}\cdot p^-(x^-)}{Z_\beta} \text{d}{x^-} 
&\leq  \frac{e^{\beta  (M -\varepsilon)} }{Z_\beta} \int_{\mathcal X} \mathbf{1} \left  \{  f(x)^T f(x^-) < M -\varepsilon \right \} p^-(x^-)\text{d}{x^-}  \\
&=  \frac{e^{\beta  (M -\varepsilon)} }{Z_\beta}  \mathbb{P}_{x^- \sim p^-} ( \mathcal B_\varepsilon )  \\
&\leq  \frac{e^{\beta  (M -\varepsilon)} }{Z_\beta}  
\end{align*}
\vspace{-10pt}
Note that
\vspace{-5pt}

 \[Z_\beta = \int_{ \mathcal X} e^{\beta  f(x)^T f(x^-)} p^-(x^-) \text{d}{x^-} \geq e^{\beta (M-\varepsilon/2)}  \mathbb{P}_{x^- \sim p^-} (  f(x)^T f(x^-) \geq  M -\varepsilon/2  ).  \]
 
 By the definition of $M=M(x)$ the probability $\rho_\varepsilon =  \mathbb{P}_{x^- \sim p^-} (  f(x)^T f(x^-) \geq  M -\varepsilon/2  ) > 0$, and we may therefore bound,
 \vspace{-10pt}
 
 \begin{align*}
\mathbb{P}_{x^- \sim q^-_\beta}(  \mathcal B_\varepsilon) &=  \frac{e^{\beta  (M -\varepsilon)} }{e^{\beta (M-\varepsilon/2)} \rho_\varepsilon} \\&= e^{ -\beta \varepsilon/2} / \rho_\varepsilon \\
& \longrightarrow 0 \;  \text{       as       } \; \beta \rightarrow \infty.
\end{align*}

We may therefore take $\beta$ to be sufficiently big so as to make $\mathbb{P}_{x^- \sim q^-_\beta}(  \mathcal B_\varepsilon)  \leq \varepsilon$ and therefore  $E_\beta(x) \leq 3\varepsilon $. In other words,   $E_\beta(x) \longrightarrow 0  $ as $\beta \rightarrow \infty$.
\end{proof} 

\vspace{-5pt}

\subsection{Optimal Embeddings on the Hypersphere for Worst-Case Negative Samples}\label{appendix: second proof}

In order to study properties of global optima of the contrastive objective using the adversarial worst case hard sampling distribution recall that we have the following limiting objective,

\begin{equation}
\mathcal L_\infty (f,q) =  \mathbb{E}_{\substack{x \sim p \\ x^+ \sim p_x^+ }} \left [-\log \frac{e^{f(x)^T f(x^+)}}{ \mathbb{E}_{x^- \sim q_\beta} [e^{f(x)^T f(x^-)}] } \right ] .
\end{equation}

We may separate the logarithm of a quotient into the sum of two terms plus a constant,
\[ \mathcal L_\infty(f,q) = \mathcal L _\text{align}(f) +\mathcal L _\text{unif}(f, q)  - 1/t^2 \]

\vspace{-5pt}
where $ \mathcal L _\text{align}(f)  = \mathbb{E}_{x,x^+ } \| f(x)-f(x^+)\|^2 /2$ and $\mathcal L _\text{unif}(f, q) = \mathbb{E}_{x \sim p} \log \mathbb{E}_{x^- \sim q} e^{f(x)^\top f(x^-)}$. Here we have used the fact that $f$ lies on the boundary of the hypersphere of radius $1/t$, which gives us the following equivalence between inner products and squared Euclidean norm, 

\begin{equation}\label{eqn: inner product MSE} 
2/t^2  - 2f(x)^\top f(x^+) = \| f(x)\|^2 + \| f(x^+)\|^2   - 2f(x)^\top f(x^+) = \| f(x)-f(x^+)\|^2.
\end{equation}

Taking supremum  to obtain $\mathcal L_\infty^*(f) = \sup_{q \in \Pi} \mathcal L_\infty(f,q)$ we find that the second expression simplifies to, 
\[\mathcal L _\text{unif}^*(f) = \sup_{q \in \Pi} \mathcal L _\text{unif}(f,q) = \mathbb{E}_{x \sim p} \log \sup_{x^- \nsim x}e^{f(x)^\top f(x^-)} = \mathbb{E}_{x \sim p}  \sup_{x^- \nsim x} f(x)^\top f(x^-).\]

Using Eqn. (\ref{eqn: inner product MSE}), this can be re-expressed as,

\begin{equation}\label{eqn: inner product MSE loss} 
\mathbb{E}_{x \sim p}  \sup_{x^- \nsim x} f(x)^\top f(x^-) = - \mathbb{E}_{x \sim p}  \inf_{x^- \nsim x} \| f(x)-f(\xm)\|^2 /2 +1/t^2.
\end{equation}

The forthcoming theorem exactly characterizes the global optima of $\min_f \mathcal L^*_\infty (f)$ \\

\begin{thm}
Suppose the downstream task is classification (i.e. $\mathcal C$ is finite), and let $ \mathcal{L}_\infty^* ( f) = \sup_{q \in \Pi} \mathcal{L}_\infty ( f, q)$ . The infimum $ \inf_{f: \; \text{measurable} } \mathcal{L}_\infty^* ( f)$  is attained, and any $f^*$ achieving the global minimum is such that $f^*(x) = f^*(x^+)$ almost surely. Furthermore, letting $\bold v_c = f^*(x)$ for any $x$ such that $h(x)=c$ (so $\bold v_c$ is well defined up to a set of $x$ of measure zero), $f^*$ is characterized as being any solution to the following ball-packing problem,

\vspace{-10pt}
\begin{equation} \max_{ \{ \bold v_c \in \mathbb{S}^{d-1}/t \}_{c \in \mathcal C}} \sum_{c \in \mathcal C} \rho(c) \cdot \min_{c' \neq c} \| \bold v_c - \bold v_{c'} \| ^2 . 
\end{equation}
\label{thm: gobal opt}
\end{thm}

\begin{proof}
Any minimizer of $ \mathcal L _\text{align}(f)  $ has the property that $f(x) = f(x^+)$ almost surely. So, in order to prove the first claim, it suffices to show that there exist functions $f \in \arg \inf_f \mathcal L _\text{unif}^*(f)$ for which $f(x) = f(x^+)$ almost surely.  This is because, at that point, we have shown that  $ \argmin_f \mathcal L _\text{align}(f)$  and $ \argmin_f \mathcal L _\text{unif}^*(f)$  intersect, and therefore any solution of $\mathcal L_\infty^*(f) = \mathcal L _\text{align}(f) +\mathcal L _\text{unif}^*(f)$ must lie in this intersection.

To this end, suppose that  $f \in  \argmin_f \mathcal L _\text{unif}^*(f)$ but that $f(x) \neq f(x^+)$ with non-zero probability. We shall show that we can construct a new embedding $\hat{f}$ such that  $f(x) = f(x^+)$  almost surely, and $\mathcal L _\text{unif}^*(\hat{f}) \leq  \mathcal L _\text{unif}^*(f)$. Due to Eqn. (\ref{eqn: inner product MSE loss}) this last condition is equivalent to showing,
\vspace{-10pt}

\begin{equation}\label{ref: eqn mnimizer configuration eqn}  \mathbb{E}_{x\sim p} \inf_{\xm \nsim x} \| \hat{f}(x) - \hat{f}(x^-) \|^2 \geq  \mathbb{E}_{x\sim p} \inf_{\xm \nsim x} \| f(x) - f(x^-) \|^2.  
\end{equation}
\vspace{-10pt}

Fix a $c \in \mathcal C$, and let $x_c \in \argmax_{x : h(x) =c} \inf_{x^- \nsim x}  \| f(x) - f(x^-) \|^2$. The maximum is guaranteed to be attained, as we explain now. Indeed we know the maximum is attained at some point in the closure $ \partial \{x : h(x) =c\} \cup  \{x : h(x) =c\}$. Since $\mathcal X$ is compact and connected, any point $\bar{x} \in \partial \{x : h(x) =c\} \setminus  \{x : h(x) =c\}$ is such that  $\inf_{x^- \nsim \bar{x}}  \| f(\bar{x}) - f(x^-) \|^2 = 0 $ since $\bar{x}$ must belong to $\{x : h(x) =c'\}$ for some other $c'$. Such an $\bar{x}$ cannot be a solution unless all points in $ \{x : h(x) =c\}$ also achieve $0$, in which case we can simply take $x_c$ to be a point in the interior of $ \{x : h(x) =c\}$. 

Now, define $\hat{f}(x) = f(x_c)$ for any $x$ such that $h(x)=c$ and $\hat{f}(x) = f(x)$ otherwise. Let us first aim to show that Eqn. (\ref{ref: eqn mnimizer configuration eqn})  holds for this $\hat{f}$. Let us begin to expand the left hand side of Eqn. (\ref{ref: eqn mnimizer configuration eqn}),
\vspace{-10pt}

\begin{align}
\mathbb{E}_{x\sim p} \inf_{\xm \nsim x} \| & \hat{f}(x) - \hat{f}(x^-) \|^2 \nonumber \\ &= \mathbb{E}_{\hat{c}\sim \rho} \mathbb{E}_{x \sim p(\cdot | \hat{c})} \inf_{\xm \nsim x} \| \hat{f}(x) - \hat{f}(x^-) \|^2   \nonumber \\
&= \rho(c)  \mathbb{E}_{x \sim p(\cdot | c)} \inf_{\xm \nsim x} \| \hat{f}(x) - \hat{f}(x^-) \|^2 & \nonumber \\
&  \quad  \quad  \quad+ (1-\rho(c) ) \mathbb{E}_{\hat{c}\sim \rho(\cdot | \hat{c} \neq c)}\mathbb{E}_{x \sim p(\cdot | \hat{c})} \inf_{\xm \nsim x} \| \hat{f}(x) - \hat{f}(x^-) \|^2 \nonumber \\
&= \rho(c)  \mathbb{E}_{x \sim p(\cdot | c)} \inf_{\xm \nsim x} \| f(x_c) - f(x^-) \|^2 \nonumber  \\
 &  \quad  \quad  \quad+ (1-\rho(c) ) \mathbb{E}_{\hat{c}\sim \rho(\cdot | \hat{c} \neq c)}\mathbb{E}_{x \sim p(\cdot | \hat{c})} \inf_{\xm \nsim x} \| \hat{f}(x) - \hat{f}(x^-) \|^2 \nonumber \\
 &= \rho(c)   \inf_{\xm \nsim x_c} \| f(x_c) - f(x^-) \|^2   \nonumber \\
 &  \quad  \quad  \quad+ (1-\rho(c) ) \mathbb{E}_{\hat{c}\sim \rho(\cdot | \hat{c} \neq c)}\mathbb{E}_{x \sim p(\cdot | \hat{c})} \inf_{h(\xm) \neq \hat{c} } \| \hat{f}(x) - \hat{f}(x^-) \|^2 \label{eqn: last line}
\end{align}

By construction, the first term can be lower bounded by $  \inf_{\xm \nsim x_c} \| f(x_c) - f(x^-) \|^2 \geq  \mathbb{E}_{x\sim p(\cdot | c)}   \inf_{h(\xm) \neq c} \| f(x) - f(x^-) \|^2$ for any $x$ such that $h(x)=c$. To lower bound the second term, consider any fixed $\hat{c} \neq c$ and $x \sim p(\cdot | \hat{c})$ (so $h(x)=\hat{c}$). Define the following two subsets of the input space $\mathcal X$
\vspace{-5pt}

\begin{align*}
\mathcal A = \{ f(\xm)   : f(\xm) \neq  \hat{c} \; \text{ for } \; \xm \in \mathcal X\} 
\qquad
\widehat{\mathcal A}=  \{ f(\xm) \in \mathcal X : \hat{f}(\xm) \neq \hat{c}  \; \text{ for } \; \xm \in \mathcal X \} .
\end{align*}
\vspace{-5pt}

Since by construction the range of $\hat{f}$ is a subset of the range of $f$, we know that $\widehat{\mathcal A} \subseteq \mathcal A$. Combining this with the fact that $\hat{f}(x) = f(x)$ whenever $h(x)=\hat{c} \neq c$ we see,

\begin{align*}
\inf_{h(\xm) \neq \hat{c} } \| \hat{f}(x) - \hat{f}(x^-) \|^2 &=  \inf_{h(\xm) \neq \hat{c} } \| f(x) - \hat{f}(x^-) \|^2 \\
&= \inf_{ u \in \widehat{\mathcal A} } \| f(x) -u \|^2 \\
&\geq \inf_{ u \in \mathcal A } \| f(x) -u \|^2 \\
& = \inf_{h(\xm) \neq \hat{c} } \| f(x) - f(x^-) \|^2
\end{align*}

Using these two lower bounds we may conclude that Eqn. (\ref{eqn: last line}) can be lower bounded by,
\vspace{-10pt}

\begin{align*}
\rho(c)  \mathbb{E}_{x\sim p(\cdot | c)}   \inf_{h(\xm) \neq c} \| f(x) - f(x^-) \|^2   + (1-\rho(c) ) \mathbb{E}_{\hat{c}\sim \rho(\cdot | \hat{c} \neq c)}\mathbb{E}_{x \sim p(\cdot | \hat{c})} \inf_{h(\xm) \neq \hat{c} } \| f(x) - f(x^-) \|^2
\end{align*}

which equals $\mathbb{E}_{x\sim p} \inf_{\xm \nsim x} \| f(x) - f(x^-) \|^2$. We have therefore proved Eqn. (\ref{ref: eqn mnimizer configuration eqn}). To summarize the current progress; given an embedding $f$ we have constructed a new embedding $\hat{f}$ that attains lower $\mathcal L _\text{unif}$ loss and which is constant on $x$ such that $\hat{f}$ is constant on  $\{ x: h(x)=c\}$. Enumerating $\mathcal C = \{ c_1 , c_2 \ldots , c_{\abs{\mathcal C}} \}$, we may repeatedly apply the same argument to construct a sequence of embeddings $f_1, f_2, \ldots , f_{\abs{\mathcal C}}$ such that $f_i$ is constant on each of the following sets $\{ x: h(x)=c_j\}$ for $j \leq i$ . The final embedding in the sequence $f^* =  f_{\abs{\mathcal C}}$ is such that  $\mathcal L _\text{unif}^*(f^*) \leq  \mathcal L _\text{unif}^*(f)$ and therefore $f^*$ is a minimizer. This embedding is constant on each of $\{ x: h(x)=c_j\}$ for $j = 1, 2, \ldots , \abs{\mathcal C}$. In other words, $f^*(x) = f^*(x^+)$ almost surely. We have proved the first claim. 

Obtaining the second claim is a matter of manipulating $\mathcal L_\infty^*(f^*)$. Indeed, we know that $\mathcal L _\infty^*(f^*)  = \mathcal L _\text{unif}^*(f^*) -1/t^2$ and defining $\bold v_c = f^*(x)=f(x_c)$ for each $c \in \mathcal C$, this expression is minimized if and only if $f^*$ attains,

\begin{align*}
\max_f \mathbb{E}_{x \sim p}  \inf_{\xm \nsim x} \| f(x) - f(\xm) \| ^2 & = \max_f \mathbb{E}_{c \sim \rho }   \mathbb{E}_{x \sim p(\cdot | c)}  \inf_{h(\xm) \neq c} \| f(x) - f(\xm) \| ^2 \\
& = \max_f \sum_{c \in \mathcal C} \rho(c) \cdot \inf_{h(\xm) \neq c} \| f(x) - f(\xm) \| ^2 \\
&=  \max_{ \{ \bold v_c \in \mathbb{S}^{d-1} /t \}_{c \in \mathcal C}} \sum_{c \in \mathcal C} \rho(c) \cdot \min_{c' \neq c} \| \bold v_c - \bold v_{c'} \| ^2 
\end{align*}

where the final equality inserts $f^*$ as an optimal $f$ and reparameterizes the maximum to be over the set of vectors $ \{ \bold v_c \in \mathbb{S}^{d-1} /t \}_{c \in \mathcal C}$. 
\end{proof}

\subsection{Downstream Generalization}\label{generalization theorem}

\begin{thm}
Suppose $\rho$ is uniform on $\mathcal C$ and $f$ is such that $\mathcal L_\infty ^*(f) - \inf_{\bar{f} \; \text{measurable}} \mathcal L_\infty ^*(\bar{f}) \leq \varepsilon$ with $\varepsilon \leq 1$. Let $ \{ \bold v^*_c \in \mathbb{S}^{d-1}/t \}_{c \in \mathcal C}$ be a solution to Problem \ref{eqn: repulsion eqn}, and define $\xi = \min_{c,c^- : c \neq c^-} \norm{\bold v^*_c -\bold v^*_{c-} } > 0$. Then there exists a set of vectors $ \{ \bold v_c \in \mathbb{S}^{d-1}/t \}_{c \in \mathcal C}$ such that the following 1-nearest neighbor classifier, 

\vspace{-10pt}
\begin{equation*} 
\hat{h}(x) = \hat{c} ,\quad\text{where}\quad \hat{c} = \arg\min_{\bar{c} \in \mathcal C} \norm{f(x) - \bold v_{\bar{c}} } \quad\text{(ties broken arbitrarily)}
\end{equation*}

achieves misclassification risk,
\begin{equation*} 
\mathbb P(\hat{h}(x) \neq c) \leq \frac{8\varepsilon}{(\xi^2 - 2\abs{\mathcal C} (1+1/t)\varepsilon^{1/2})^2}
\end{equation*}
\end{thm}

\begin{proof}
To begin, using the definition of $\hat{h}$ we know that for any $0<\delta <\xi$,

\begin{align*}
\mathbb{P}_{x,c}( \hat{h}(x) = c) &= \mathbb{P}_{x,c} \left ( \norm{f(x) -  \bold v_c} \leq \min_{c^-:c^-\neq c} \norm{f(x) -  \bold v_{c^-}} \right ) \\
& \geq  \mathbb{P}_{x,c} \left ( \norm{f(x) -  \bold v_c} \leq \delta, \quad\text{and}\quad \delta \leq \min_{c^-:c^-\neq c} \norm{f(x) -  \bold v_{c^-}} \right ) \\
&\geq 1 -   \mathbb{P}_{x,c} \left ( \norm{f(x) -  \bold v_c} > \delta \right)  -   \mathbb{P}_{x,c}  \big (\min_{c^-:c^-\neq c} \norm{f(x) -  \bold v_{c^-}} < \delta\big )
\end{align*}

So to prove the result, our goal is now to bound these two probabilities. To do so, we use the bound on the excess risk. Indeed, combining the fact $\mathcal L_\infty ^*(f) - \inf_{\bar{f} \; \text{measurable}} \mathcal L_\infty ^*(\bar{f}) \leq \varepsilon$ with the notational rearrangements before Theorem \ref{thm: gobal opt} we observe that $\mathbb{E}_{x,x^+} \norm{f(x) - f(x^+)}^2 \leq  2\varepsilon$.

We have,

\begin{align*}
2\varepsilon \geq \mathbb{E}_{x,x^+} \norm{f(x) - f(x^+)}^2 
&= \mathbb{E}_{c \sim \rho} \mathbb{E}_{x^+ \sim p(\cdot | c)} \mathbb{E}_{x \sim p(\cdot | c)}  \norm{f(x) - f(x^+)}^2. 
\end{align*}
\vspace{-10pt}

For fixed $c,x^+$, let $x_c \in \arg\min_{\overline{\{x^+: h(x^+) = c\}}} \mathbb{E}_{x \sim p(\cdot | c)} \norm{f(x) - f(x^+)}^2$ where we extend the minimum to be over the closure, a compact set, to guarantee it is attained. Then we have 

\begin{align*}
 2\varepsilon \geq \mathbb{E}_{c \sim \rho} \mathbb{E}_{x^+ \sim p(\cdot | c)} \mathbb{E}_{x \sim p(\cdot | c)}  \norm{f(x) - f(x^+)}^2 \geq  \mathbb{E}_{c \sim \rho} \mathbb{E}_{x \sim p(\cdot | c)}   \norm{f(x) - \bold v_c}^2
\end{align*}

where we have now defined $\bold v_c =  f(x_c) $ for each $c \in \mathcal C$. Note in particular that $\bold v_c$ lies on the surface of the hypersphere $\mathbb{S}^{d-1}/t$. This enables us to obtain the follow bound using Markov's inequality, 

\begin{align*}
\mathbb{P}_{x,c} \left ( \norm{f(x) -  \bold v_c} > \delta \right) &= \mathbb{P}_{x,c} \left ( \norm{f(x) -  \bold v_c}^2 > \delta^2 \right)  \\
&\leq  \frac{\mathbb{E}_{x,c}   \norm{f(x) - \bold v_c}^2}{\delta^2} \\
& \leq \frac{2\varepsilon}{\delta^2}.
\end{align*}

so it remains still to bound $\mathbb{P}_{x,c}  \big (\min_{c^-:c^-\neq c} \norm{f(x) -  \bold v_{c^-}} < \delta\big )$. Defining $\xi' = \min_{c,c^- : c \neq c^-} \norm{\bold v_c -\bold v_{c-} }$, we have the following fact (proven later).

\paragraph{Fact (see lemma \ref{aux lemma 1}):} $\xi' \geq \sqrt{\xi^2 - 2 \abs{\mathcal C} (1+1/t)\sqrt{\varepsilon}}$.

Using this fact we are able to get control over the tail probability as follows,  
\begin{align*}
\mathbb{P}_{x,c}  \left (\min_{c^-:c^-\neq c} \norm{f(x) -  \bold v_{c^-}} < \delta \right) &\leq \mathbb{P}_{x,c} \left ( \norm{f(x) -  \bold v_c} >  \xi' - \delta \right ) \\
 &\leq \mathbb{P}_{x,c} \left ( \norm{f(x) -  \bold v_c} >  \xi - \sqrt{\xi^2 - 2 \abs{\mathcal C} (1+1/t)\varepsilon^{1/2} } -  \delta \right ) \\
&= \mathbb{P}_{x,c} \left ( \norm{f(x) -  \bold v_c}^2 > (\sqrt{\xi^2 - 2 \abs{\mathcal C} (1+1/t)\varepsilon^{1/2} }  - \delta )^2 \right ) \\
& \leq \frac{2\varepsilon}{(\sqrt{\xi^2 - 2 \abs{\mathcal C} (1+1/t)\varepsilon^{1/2} }  - \delta )^2}.
\end{align*}

where this inequality holds for for any $0 \leq \delta \leq \sqrt{\xi^2 - 2 \abs{\mathcal C} (1+1/t)\varepsilon^{1/2} }$. 
 
 Gathering together our tail probability bounds we find that $\mathbb{P}_{x,c}( \hat{h}(x) = c) \geq 1 - \frac{2\varepsilon}{\delta^2} - \frac{2\varepsilon}{(\sqrt{\xi^2 - 2 \abs{\mathcal C} (1+1/t)\varepsilon^{1/2}}  - \delta )^2}$ for any $0 \leq \delta \leq \sqrt{\xi^2 - 2 \abs{\mathcal C} (1+1/t)\varepsilon^{1/2}}$. That is,
 
 \vspace{-5pt}
 \begin{align*}
\mathbb{P}_{x,c}( \hat{h}(x) \neq c) \leq \frac{2\varepsilon}{\delta^2} + \frac{2\varepsilon}{(\sqrt{\xi^2 - 2 \abs{\mathcal C} (1+1/t)\varepsilon^{1/2}}  - \delta )^2}
\end{align*}
\vspace{-5pt}

Since this holds for any $0 \leq \delta \leq \sqrt{\xi^2 - 2 \abs{\mathcal C} (1+1/t)\varepsilon^{1/2}}$,
\vspace{-10pt}

 \begin{align*}
\mathbb{P}_{x,c}( \hat{h}(x) \neq c) \leq \min_{0 \leq \delta \leq \sqrt{\xi^2 - 2 \abs{\mathcal C} \varepsilon}} \bigg \{  \frac{2\varepsilon}{\delta^2} + \frac{2\varepsilon}{(\sqrt{\xi^2 - 2 \abs{\mathcal C} (1+1/t)\varepsilon^{1/2}}  - \delta )^2} \bigg \}.
\end{align*}

Elementary calculus shows that the minimum is attained at $\delta =\frac{\sqrt{\xi^2 - 2 \abs{\mathcal C} (1+1/t)\varepsilon^{1/2}}}{2}$. Plugging this in yields the final bound, 

\vspace{-10pt}
\begin{equation*} 
\mathbb P(\hat{h}(x) \neq c) \leq \frac{8\varepsilon}{(\xi^2 - 2\abs{\mathcal C}(1+1/t)\varepsilon^{1/2})^2}.
\end{equation*}

\end{proof}

\begin{lemma}
Consider the same setting as introduced in Theorem \ref{generalization thm}. In particular define 

 \begin{align*}
\xi' = \min_{c,c^- : c \neq c^-} \norm{\bold v_c -\bold v_{c-} },
\qquad
\xi = \min_{c,c^- : c \neq c^-} \norm{\bold v^*_c -\bold v^*_{c-} }.
\end{align*}

where $ \{ \bold v^*_c \in \mathbb{S}^{d-1}/t \}_{c \in \mathcal C}$ is a solution to Problem \ref{eqn: repulsion eqn}, and $ \{ \bold v_c \in \mathbb{S}^{d-1}/t \}_{c \in \mathcal C}$ is defined via $\bold v_c =  f(x_c) $ with $x_c \in \arg\min_{\overline{\{x^+: h(x^+) = c\}}} \mathbb{E}_{x \sim p(\cdot | c)} \norm{f(x) - f(x^+)}^2$  for each $c \in \mathcal C$. Then we have,  

\[\xi' \geq \sqrt{\xi^2 - 2 \abs{\mathcal C} (1+1/t)\varepsilon^{1/2} }.\]
\label{aux lemma 1}
\end{lemma}

\begin{proof}
Define,

\begin{align*}
X = \min_{c^-:c^-\neq c} \norm{\bold v_c -\bold v_{c-} }^2,
\qquad
X^* = \min_{c^-:c^-\neq c} \norm{\bold v^*_c -\bold v^*_{c-} }^2.
\end{align*}

$X$ and $X^*$ are random due to the randomness of $c \sim \rho$. We can split up the following expectation by conditioning on the event $\{ X \leq X^*\}$ and its complement,

\begin{align}\label{abs eq X X^*}
\mathbb{E}\abs{X-X^* } = \mathbb{P}(X \geq X^*) \mathbb{E}[ X -X^*]  +  \mathbb{P}(X \leq X^*) \mathbb{E}[ X^* -X] .
\end{align}

Using $\mathcal L_\infty ^*(f) - \inf_{\bar{f} \; \text{measurable}} \mathcal L_\infty ^*(\bar{f}) \leq \varepsilon$ and the notational re-writing of the objective  $\mathcal L_\infty ^*$ introduced before Theorem \ref{thm: gobal opt}, we observe the following fact, whose proof we give in a separate lemma after the conclusion of this proof.

\paragraph{Fact (see lemma \ref{aux lemma}):} $\mathbb{E}X^* - 2(1+1/t)\sqrt{\varepsilon}\leq \mathbb{E}X \leq \mathbb{E}X^*$. 

This fact implies in particular $ \mathbb{E}[ X -X^*]\leq 0$ and $ \mathbb{E}[ X^* -X] \leq 2(1+1/t)\sqrt{\varepsilon}$. Inserting both inequalities into Eqn. \ref{abs eq X X^*} we find that $\mathbb{E}\abs{X-X^* } \leq 2 (1+1/t)\sqrt{\varepsilon}$. In other words, since $\rho$ is uniform,

\begin{align*}
\frac{1}{\abs{\mathcal C} } \sum_{c \in \mathcal C} \abs{ \min_{c^-:c^-\neq c} \norm{\bold v_c -\bold v_{c-} }^2 - 
 \min_{c^-:c^-\neq c} \norm{\bold v^*_c -\bold v^*_{c-} }^2 } \leq 2 (1+1/t)\sqrt{\varepsilon}.
\end{align*}

From which we can say that for any $c \in \mathcal C$ , 

\begin{align*}
\abs{ \min_{c^-:c^-\neq c} \norm{\bold v_c -\bold v_{c-} }^2 - 
 \min_{c^-:c^-\neq c} \norm{\bold v^*_c -\bold v^*_{c-} }^2 } \leq 2 \abs{\mathcal C} (1+1/t)\sqrt{\varepsilon}.
\end{align*}

So $\min_{c^-:c^-\neq c} \norm{\bold v_c -\bold v_{c-} } \geq \sqrt{
 \min_{c^-:c^-\neq c} \norm{\bold v^*_c -\bold v^*_{c-} }^2 - 2 \abs{\mathcal C} (1+1/t)\varepsilon^{1/2} } \geq \sqrt{\xi^2 - 2 \abs{\mathcal C}(1+1/t)\varepsilon^{1/2}}$. Since this holds for any $c \in \mathcal C$ , we conclude that $\xi' \geq\sqrt{\xi^2 - 2 \abs{\mathcal C} (1+1/t)\varepsilon^{1/2} }$.
\end{proof}

\begin{lemma}\label{aux lemma}
Consider the same setting as introduced in Theorem \ref{generalization thm}. Define also,
\begin{align*}
X = \min_{c^-:c^-\neq c} \norm{\bold v_c -\bold v_{c-} }^2,
\qquad
X^* = \min_{c^-:c^-\neq c} \norm{\bold v^*_c -\bold v^*_{c-} }^2,
\end{align*}

where $\bold v_c =  f(x_c)$ with $x_c \in \arg\min_{\overline{\{x^+: h(x^+)= c\}}} \mathbb{E}_{x \sim p(\cdot | c)} \norm{f(x) - f(x^+)}^2$ for each $c \in \mathcal C$. We have, 
\[ \mathbb{E}X^* - 2(1+1/t)\sqrt{\varepsilon} \leq \mathbb{E}X \leq \mathbb{E}X^*.\]
\end{lemma}

\begin{proof}
By Theorem \ref{eqn: repulsion eqn} we know there is an $f^*$ attaining the minimum $\inf_{\bar{f} \; \text{measurable}} \mathcal L_\infty ^*(\bar{f})$ and that this $f^*$ attains $ \mathcal L^* _\text{align}(f^*) =0$, and also minimizes the uniformity term $\mathcal L _\text{unif}^*(f) $, taking the value $\mathcal L _\text{unif}^*(f^*) =  \mathbb{E}_{c \sim \rho} \max_{c^-: c^- \neq c}  \bold {v^*_c}^\top \bold v^*_{c^-}$.
Because of this we find,

\begin{align*}
\mathcal L _\text{unif}^*(f) &\leq \big ( \mathcal L _\infty^*(f) -  \mathcal L _\infty^*(f^*) \big ) + \big (  \mathcal L _\text{align}^*(f^*)  - \mathcal L _\text{align}^*(f)  \big ) + \mathcal L _\text{unif}^*(f^*) \\
&\leq \big ( \mathcal L _\infty^*(f) -  \mathcal L _\infty^*(f^*) \big )  + \mathcal L _\text{unif}^*(f^*) \\
&\leq \varepsilon + \mathcal L _\text{unif}^*(f^*) \\
&=  \varepsilon +\mathbb{E}_{c \sim \rho} \max_{c^-: c^- \neq c}  \bold {v^*_c}^\top \bold v^*_{c^-}.
\end{align*}

Since we would like to bound $\mathbb{E}_{c \sim \rho} \max_{c^-: c^- \neq c}  \bold {v_c}^\top \bold v_{c^-}$ in terms of $\mathbb{E}_{c \sim \rho} \max_{c^-: c^- \neq c}  \bold {v^*_c}^\top \bold v^*_{c^-}$, this observation means that is suffices to bound $\mathbb{E}_{c \sim \rho} \max_{c^-: c^- \neq c}  \bold {v_c}^\top \bold v_{c^-}$ in terms of $\mathcal L _\text{unif}^*(f)$. To this end, note that for a fixed $c$, and $x$ such that $h(x)=c$ we have, 

\begin{align*}
\sup_{x^- \nsim x} f(x) ^\top f(x^-) &= \sup_{x^- \nsim x}  \big \{ {\bold v_c}^\top  f(x^-) +  (f(x) -{\bold v_c})^\top f(x^-) \big \} \\
&= \sup_{x^- \nsim x}   {\bold v_c}^\top  f(x^-) -  \norm{f(x) -{\bold v_c}}/t \\
&\geq \max_{x^- \in \{ x_c\}_{c \in \mathcal C}}  {\bold v_c}^\top  f(x^-)  -  \norm{f(x) -{\bold v_c}}/t \\
&= \max_{c^- \neq c}   {\bold v_c}^\top    {\bold v_{c^-}} -  \norm{f(x) -{\bold v_c}}/t \\
\end{align*}

where the inequality follows since $ \{ x_c\}_{c \in \mathcal C}$ is a subset of the closure of $\{x^-: x^- \nsim x\}$. Taking expectations over $c,x$,

\begin{align*}
\mathcal L _\text{unif}^*(f)  &= \mathbb{E}_{x,c}\sup_{x^- \nsim x} f(x) ^\top f(x^-) \\
&\geq \mathbb{E}_{c\sim \rho} \max_{c^- \neq c}   {\bold v_c}^\top    {\bold v_{c^-}} -  \mathbb{E}_{x,c} \norm{f(x) -{\bold v_c}}/t \\
&\geq \mathbb{E}_{c\sim \rho} \max_{c^- \neq c}   {\bold v_c}^\top    {\bold v_{c^-}} -  \sqrt{\mathbb{E}_{x,c} \norm{f(x) -{\bold v_c}}^2}/t \\
&\geq \mathbb{E}_{c\sim \rho} \max_{c^- \neq c}   {\bold v_c}^\top    {\bold v_{c^-}} -  \sqrt{\varepsilon}/t .
\end{align*}

So since $\varepsilon \leq \sqrt{\varepsilon}$, we have found that

\[  \mathbb{E}_{c\sim \rho} \max_{c^- \neq c}   {\bold v_c}^\top    {\bold v_{c^-}} \leq  \sqrt{\varepsilon}/t +  \varepsilon + \mathbb{E}_{c \sim \rho} \max_{c^-: c^- \neq c}  \bold {v^*_c}^\top \bold v^*_{c^-} \leq (1+1/t)\sqrt{\varepsilon} +  \mathbb{E}_{c \sim \rho} \max_{c^-: c^- \neq c}  \bold {v^*_c}^\top \bold v^*_{c^-}.\]

Of course we also have,
\[\mathbb{E}_{c \sim \rho} \max_{c^-: c^- \neq c}  \bold {v^*_c}^\top {\bold v^*_{c^-}} =  \mathcal L _\text{unif}^*(f^*) \leq \mathbb{E}_{c \sim \rho} \max_{c^-: c^- \neq c}  \bold {v_c}^\top \bold v_{c^-}  \]

since the embedding $f(x)=\bold v_c$ whenever $h(x)=c$ is also a feasible solution. Combining these two inequalities with the simple identity $\bold x^\top \bold y = 1/t^2 - \norm{\bold x - \bold y}^2/2$ for all length $1/t$ vectors $\bold x, \bold y$, we find,

\begin{align*}
1/t^2  - \mathbb{E}_{c \sim \rho} \max_{c^-: c^- \neq c} \norm{  \bold {v^*_c} - \bold v^*_{c^-}}^2/2  &\leq 1/t^2 - \mathbb{E}_{c \sim \rho} \max_{c^-: c^- \neq c} \norm{  \bold {v_c} - \bold v_{c^-}}^2/2  \\
&\leq  1/t^2  - \mathbb{E}_{c \sim \rho} \max_{c^-: c^- \neq c} \norm{  \bold {v^*_c} - \bold v^*_{c^-}}^2 /2 + (1+1/t)\sqrt{\varepsilon}.
\end{align*}

Subtracting $1/t^2 $ and multiplying by $-2$ yields the result.

\end{proof}

\section{Graph Representation Learning}\label{appendix: graph method}

We describe in detail the hard sampling method for graphs whose results are reported in Section \ref{section: graph classification}. Before getting that point, in the interests of completeness we cover some required background details on the InfoGraph method of  \cite{sun2019infograph}. For further information see the original paper \citep{sun2019infograph}.

\subsection{Background on Graph Representations}

We observe a set of graphs $\mathbf{G} = \{ G_j \in \mathbb{G}  \}_{j=1}^n$ sampled according to a distribution $p$ over an ambient graph space $\mathbb{G}$. Each node $u$ in a graph $G$ is assumed to have features $h^{(0)}_u$ living in some Euclidean space. We consider a $K$-layer graph neural network, whose $k$-th layer iteratively computes updated embeddings for each node $v \in G$ in the following way,

\[ h^{(k)}_v = \text{COMBINE}^{(k)} \left ( h_v^{(k-1)} , \text{AGGREGATE}^{(k)}  \left (  \left \{ \left ( h_v^{(k-1)}, h_u^{(k-1)}  , e_{uv} \right )  : u \in \mathcal N (v)   \right \}  \right )   \right ) \]

where $ \text{COMBINE}^{(k)}$ and $\text{AGGREGATE}^{(k)}$ are parameterized learnable functions and $ \mathcal N (v) $ denotes the set of neighboring nodes of $v$. The $K$ embeddings for a node $u$ are collected together to obtain a single final summary embedding for $u$. As recommended by \cite{xu2018powerful} we use concatenation, $h^u = h^u(G) = \text{CONCAT} \left (\{ h^{(k)}_u \}_{k=1}^K  \right)$ to obtain an embedding in $\mathbb{R}^d$. Finally, the node representations are combined together into a length $d$ graph level embedding using a readout function,

\[ H(G) = \text{READOUT} \left (\{ h^{u} \}_{u \in G}  \right)\]

which is typically taken to be a simple permutation invariant function such as the sum or mean. The InfoGraph method aims to maximize the mutual information between the graph level embedding $H(G) $ and patch-level embeddings $h^u(G) $ using the following objective,

\[ \max_h \mathbb{E}_{G \sim p} \frac{1}{\abs{G}} \sum_{u\in G} I \left (h^u(G) ; H(G) \right )   \]

In practice the population distribution $p$ is replaced by its empirical counterpart, and the mutual information $I$ is replaced by a variational approximation $I_T$. In line with  \cite{sun2019infograph} we use the Jensen-Shannon mutual information estimator as formulated by \cite{nowozin2016f}. It is defined using a neural network discriminator $T : \mathbb{R}^{2d} \rightarrow \mathbb{R}$ as, 

\[ I_T  \left (h^u(G) ; H(G) \right ) = \mathbb{E}_{G \sim p} \left [ - \text{sp}( - T   \left (h^u(G) , H(G) \right )) \right ] - \mathbb{E}_{(G,G')  \sim p \times  p} \left [ \text{sp}(  T   \left (h^u(G) , H(G') \right )) \right ] \]

where $\text{sp}(z) = \log(1+e^z)$ denotes the softplus function. The finial objective is the joint maximization over $h$ and $T$, 

\[ \max_{\theta, \psi} \mathbb{E}_{G \sim p } \frac{1}{\abs{G}} \sum_{u\in G} I_T \left (h^u(G) ; H(G) \right )   \]

\subsection{Hard Negative Sampling for Learning Graph Representations}

In order to derive a simple modification of the NCE hard sampling technique that is appropriate for use with InfoGraph, we first provide a mildly generalized view of hard sampling. Recall that the NCE contrastive objective can be decomposed into two constituent pieces,

\[ \mathcal L (f,q) = \mathcal L_\text{align}(f) +  \mathcal L_\text{unif}(f,q) \]  

where $q$ is in fact a family of distributions $q(\xm; x)$ over $\xm$ that is indexed by the possible values of the anchor $x$. $ \mathcal L_\text{align}$ performs the role of  ``aligning'' positive pairs (embedding near to one-another), while $\mathcal L_\text{unif}$ repels negative pairs. The hard sampling framework aims to solve,

\[ \inf_f \sup_q \mathcal L (f,q). \]  

In the case of NCE loss we take,

\vspace{-10pt}
\begin{align*}
& \mathcal L_\text{align}(f)  = -  \mathbb{E}_{\substack{x \sim p \\ x^+ \sim p_x^+ }} f(x)^T f(x^+)  ,  \\
& \mathcal L_\text{unif}(f,q)  =    \mathbb{E}_{\substack{x \sim p \\ x^+ \sim p_x^+ }} \log \left \{ e^{f(x)^T f(x^+)} +Q \mathbb{E}_{x^- \sim q} [e^{f(x)^T f(x^-)}]  \right \}.
\end{align*}

View this view, we can easily adapt to the InfoGraph framework, taking

\vspace{-10pt}
\begin{align*}
& \mathcal L_\text{align}(h,T)  = -  \mathbb{E}_{G \sim p} \frac{1}{\abs{G}} \sum_{u\in G} \text{sp}( - T   \left (h^u(G) , H(G) \right )) ,  \\
& \mathcal L_\text{unif}(h,T,q)  =    -  \mathbb{E}_{G \sim p} \frac{1}{\abs{G}} \sum_{u\in G}  \mathbb{E}_{G'  \sim  q } \text{sp}(  T   \left (h^u(G) , H(G') \right )) 
\end{align*}

Denote by $\hat{p}$ the distribution over nodes $u \in \mathbb R^s$ defined by first sampling $G \sim p$, then sampling $u \in G$ uniformly over all nodes of $G$. Then these two terms can be simplified to 

\vspace{-10pt}
\begin{align*}
& \mathcal L_\text{align}(h,T)  = -  \mathbb{E}_{u \sim \hat{p}} \text{sp}( - T   \left (h^u(G) , H(G) \right )) ,  \\
& \mathcal L_\text{unif}(h,T,q)  =    -   \mathbb{E}_{(u,G') \sim \hat{p} \times  q } \text{sp}(  T   \left (h^u(G) , H(G') \right )) 
\end{align*}

At this point it becomes clear that, just as with NCE, a distribution $q^* \in \argmax_q \mathcal L (f,q)$ in the InfoGraph framework if it is supported on $\argmax_{G' \in \mathbb G} \text{sp}(  T   \left (h^u(G) , H(G') \right )) $. Although this is still hard to compute exactly, it can be approximated by,

\[ q_u^\beta(G') \propto \exp \left ( \beta T (h^u(G), H(G) ) \right ) \cdot p(G'). \]

 \begin{wrapfigure}{r}{6.5cm}
  \includegraphics[width=6.5cm]{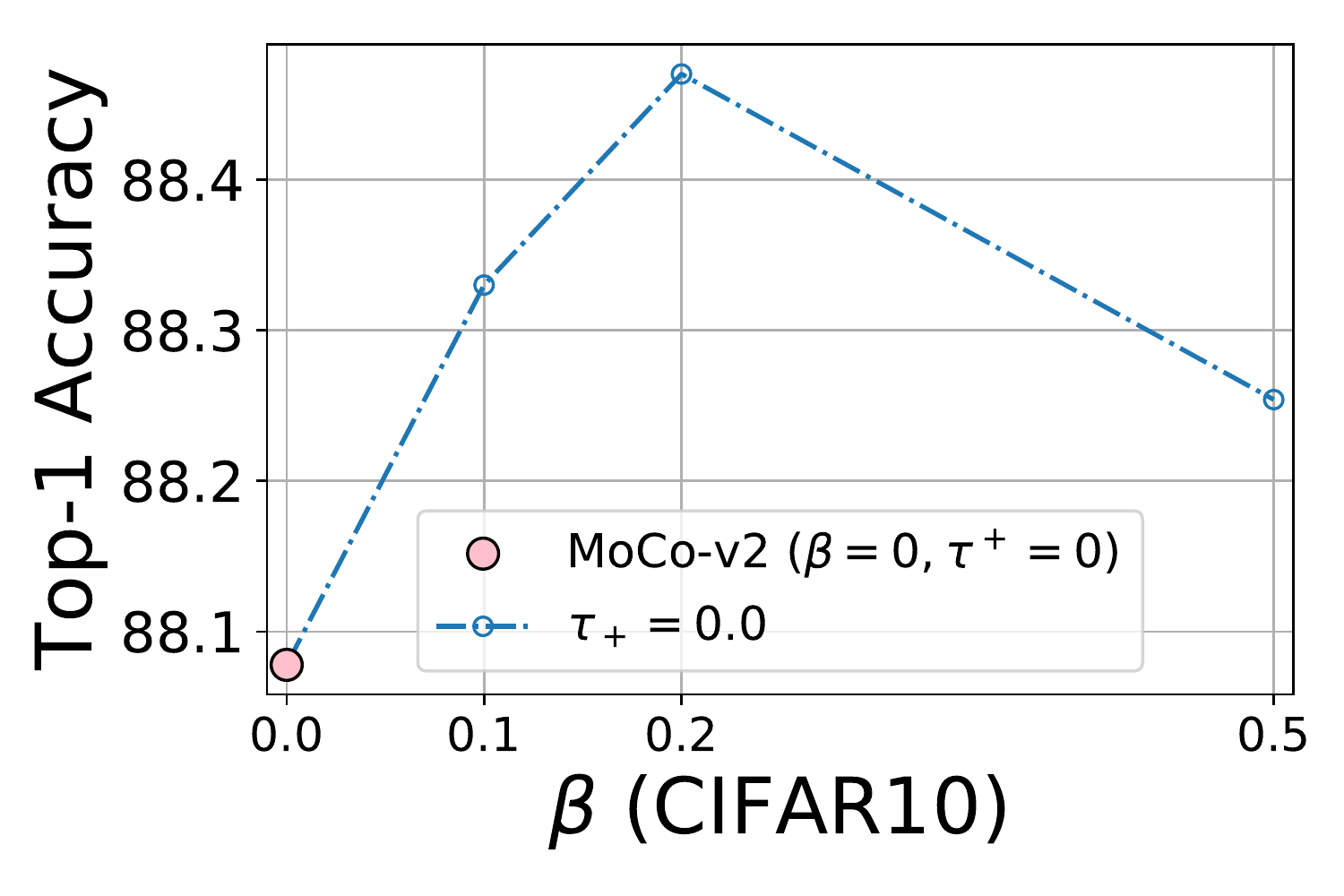}  
     \caption{ Hard negative sampling using MoCo-v2 framework. Results show that hard negative samples can still be useful when the negative memory bank is very large (in this case  $N=65536$).}
     \vspace{-30pt}
      \label{fig: moco beta}
\end{wrapfigure}
\section{Additional Experiments}\label{sec: app ablations}

\subsection{Hard negatives with large batch sizes}\label{sec: moco}

The vision experiments in the main body of the paper are all based off the SimCLR framework \citep{chen2020simple}. They use a relatively small batch size (up to $512$). In order to test whether our hard negatives sampling method can help when the negative batch size is very large, we also run experiments using MoCo-v2 with standard negative memory bank size $N=65536$ \citep{he2020momentum,chen2020improved}. We adopt the official MoCo-v2 code\footnote{  \url{https://github.com/facebookresearch/moco}}. Embeddings are trained for $200$ epochs, with batch size $128$. Figure \ref{fig: moco beta} summarizes the results. We find that hard negative sampling can still improve the generalization of embeddings trained on CIFAR10: MoCo-v2 attains linear readout accuracy of 88.08\%, and MoCo-v2  with hard negatives ($\beta=0.2$, $\tau^+=0$) attains 88.47\%.

 \begin{wrapfigure}{r}{6.5cm}
  \includegraphics[width=6.5cm]{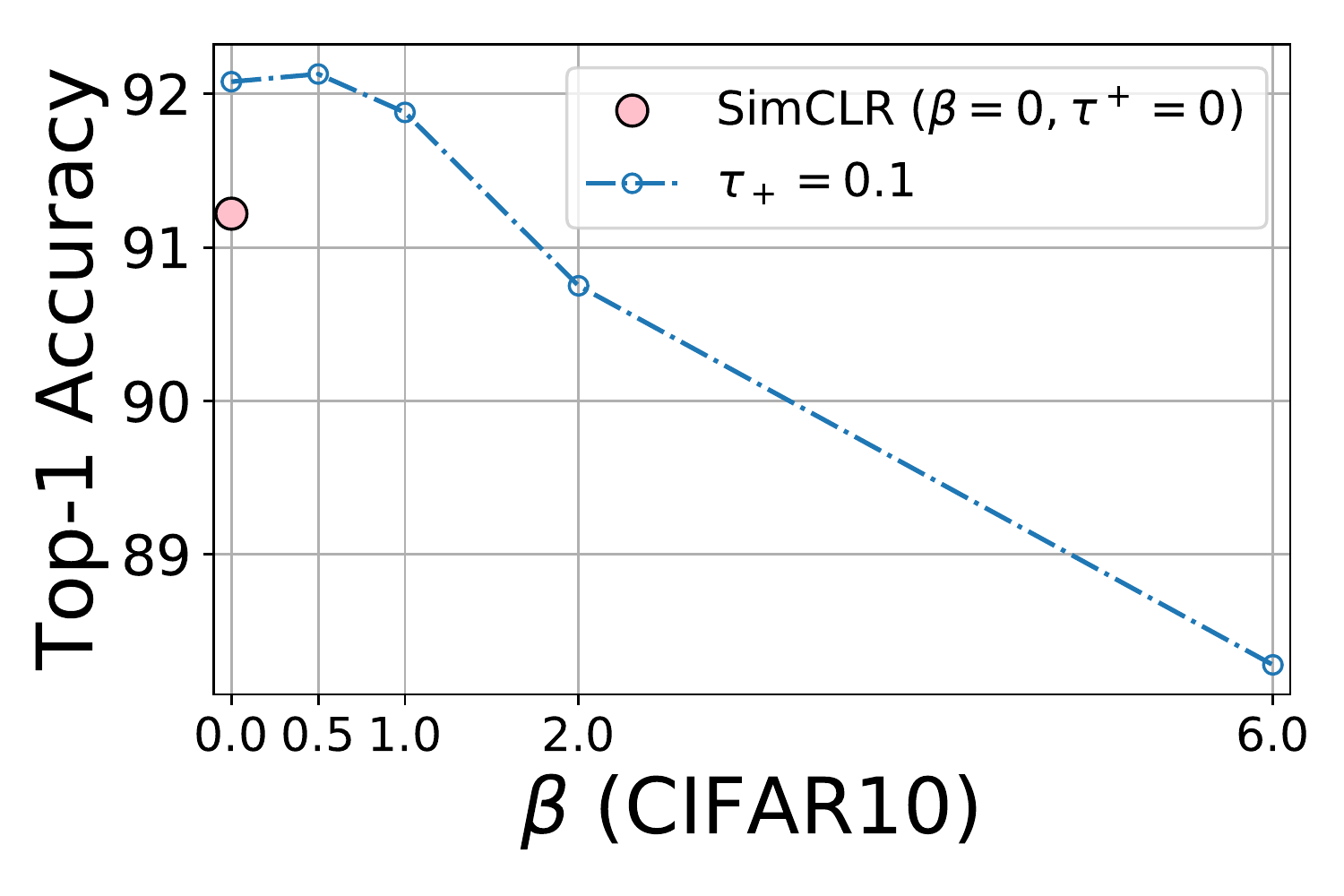}  \\
     \caption{ The effect of varying concentration parameter $\beta$ on linear readout accuracy for CIFAR10. (Complements the left and middle plot from Figure \ref{fig:stl ablation}.)  }
      \vspace{-15pt}
      \label{fig: cifar10 beta}
\end{wrapfigure}
\subsection{Ablations}
To study the affect of varying the concentration parameter $\beta$ on the learned embeddings Figure \ref{fig:cifar100 histogram ablation} plots cosine similarity histograms of pairs of similar and dissimilar points. The results show that for $\beta$ moving from $0$ through $0.5$ to $2$ causes both the positive and negative similarities to gradually skew left. In terms of downstream classification, an important property is the \emph{relative} difference in similarity between positive and negative pairs. In this case $\beta=0.5$ find the best balance (since it achieves the highest downstream accuracy). When $\beta$ is taken very large ($\beta=6$), we see a change in conditions. Both positive and negative pairs are assigned higher similarities in general. Visually it seems that the positive and negative histograms for $\beta=6$  overlap a lot more than for smaller values, which helps explain why the linear readout accuracy is lower for $\beta=6$ .

Figure \ref{fig: real fig1} gives real examples of hard vs. uniformly sampled negatives. Given an anchor $x$ (a monkey) and trained embedding  $f$ (trained on STL10 using standard SimCLR for $400$ epochs), we sample a batch of $128$ images. The top row shows the ten negatives $\xm$ that have the largest inner product $f(x)^\top f(\xm)$, while the bottom row is a random sample from from the same batch. Negatives with the largest inner product with the anchor correspond to the items in the batch are the most important terms in the objective since they are given the highest weighting by $q_\beta^-$. Figure \ref{fig: real fig1} shows that ``real'' hard negatives are conceptually similar to the idea as proposed in Figure 1: hard negatives are semantically similar to the anchor, possessing various similarities, including color (browns and greens), texture (fur), and objects (animals vs machinery).

\begin{figure}[h]
\vspace{+10pt}
    \includegraphics[width=\textwidth]{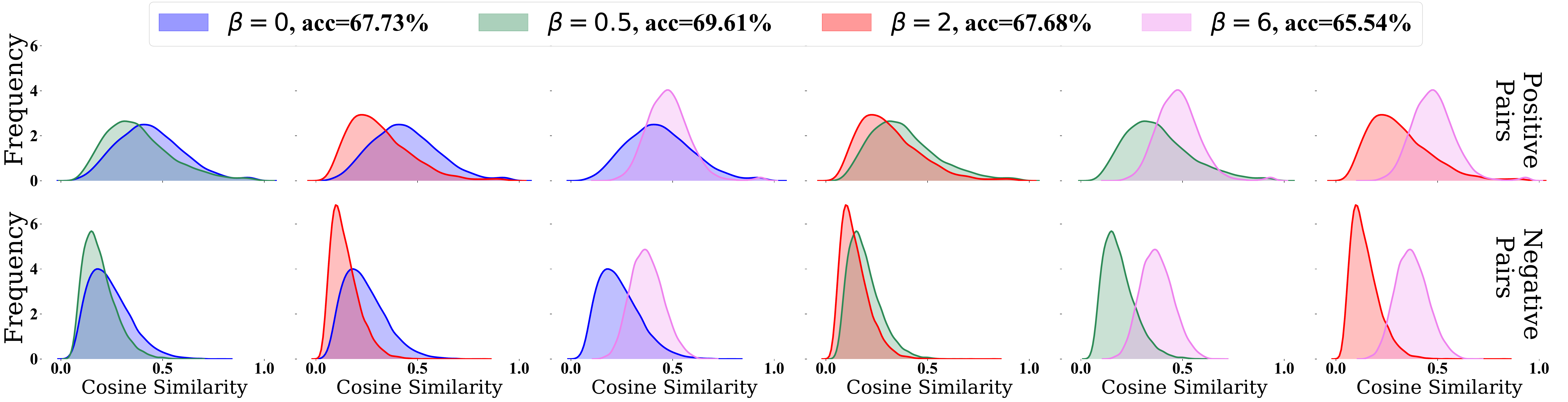}
    \caption{Histograms of cosine similarity of pairs of points with different label (bottom) and same label (top) for embeddings trained on CIFAR100 with  different values of $\beta$. Histograms overlaid pairwise to allow for easy comparison. }
    \label{fig:cifar100 beta histogram ablation}
\end{figure}

\begin{figure}[h]
    \includegraphics[width=\textwidth]{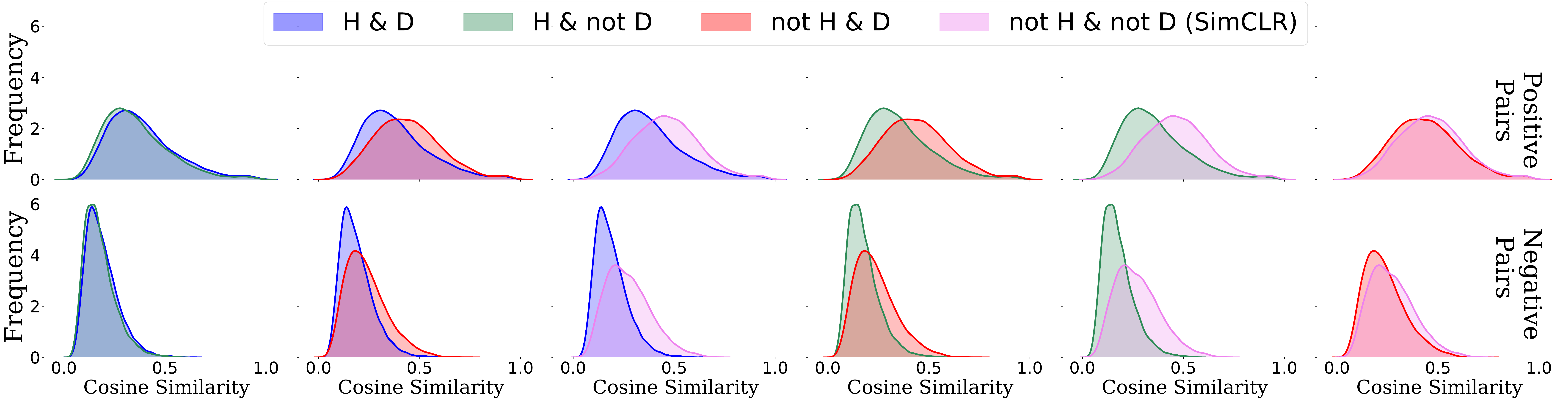}
    \caption{Histograms of cosine similarity of pairs of points with the same label (top) and different labels (bottom) for embeddings trained on CIFAR100 with four different objectives. H$=$Hard Sampling, D$=$Debiasing.  Histograms overlaid pairwise to allow for convenient comparison.}
    \label{fig:cifar100 histogram ablation}
\end{figure}

\begin{figure}[h]
    \includegraphics[width=\textwidth]{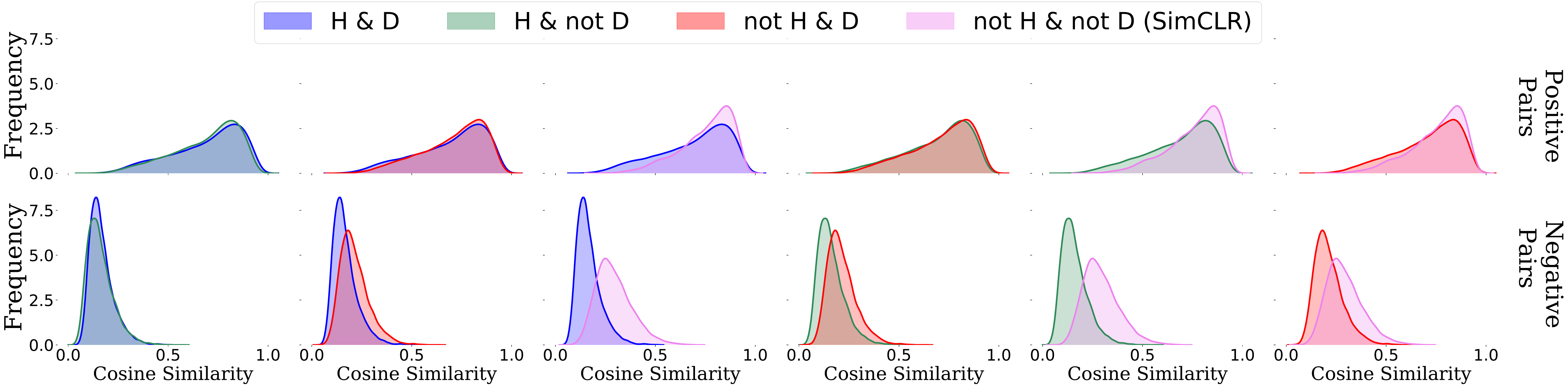}
    \caption{Histograms of cosine similarity of pairs of points with the same label (top) and different labels (bottom) for embeddings trained on CIFAR10 with four different objectives. H$=$Hard Sampling, D$=$Debiasing.  Histograms overlaid pairwise to allow for convenient comparison.}
    \label{fig:cifar10 histogram ablation}
\end{figure}

 \begin{figure}[h]
  \centering
  \includegraphics[width=6.5cm]{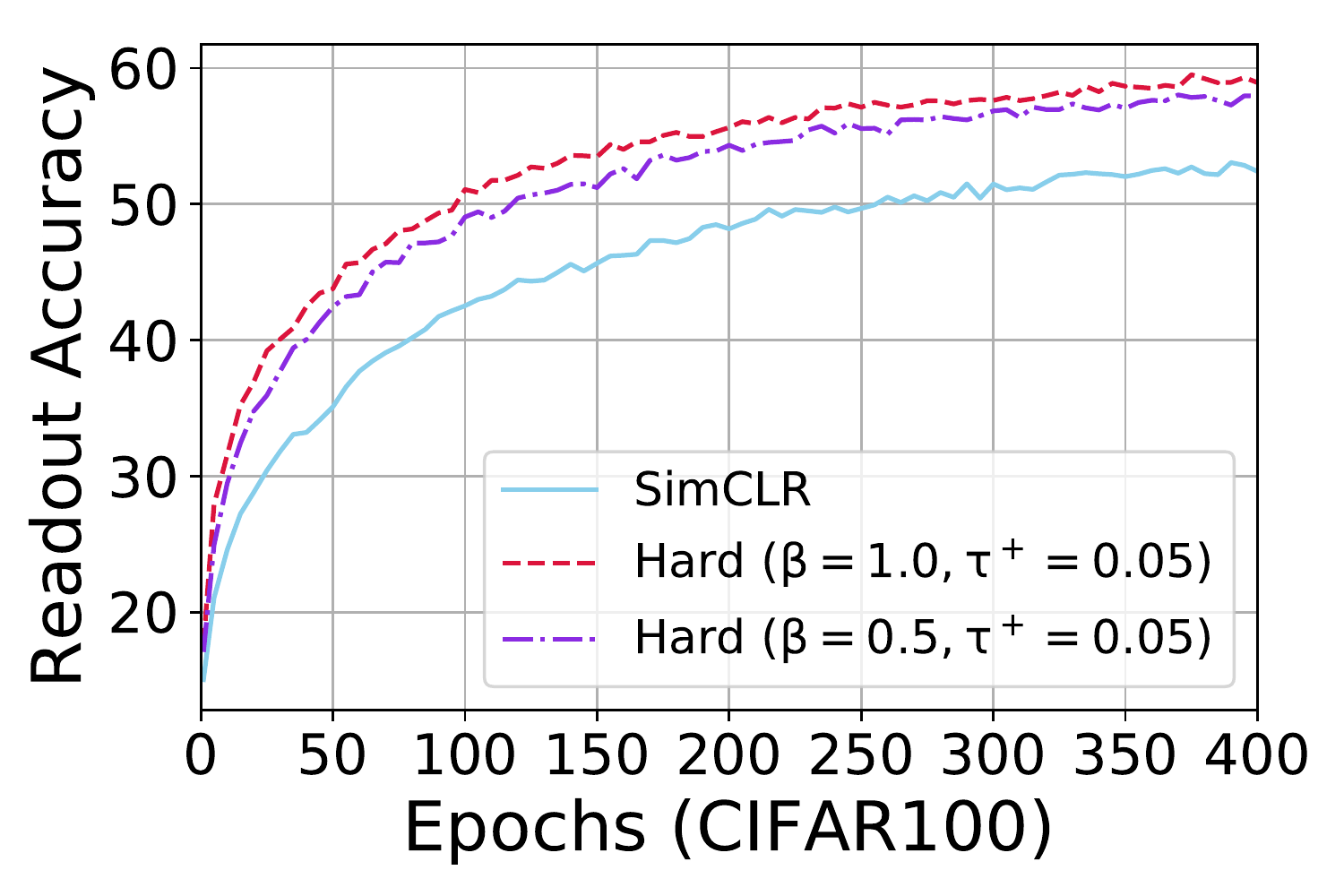}  
  \includegraphics[width=6.5cm]{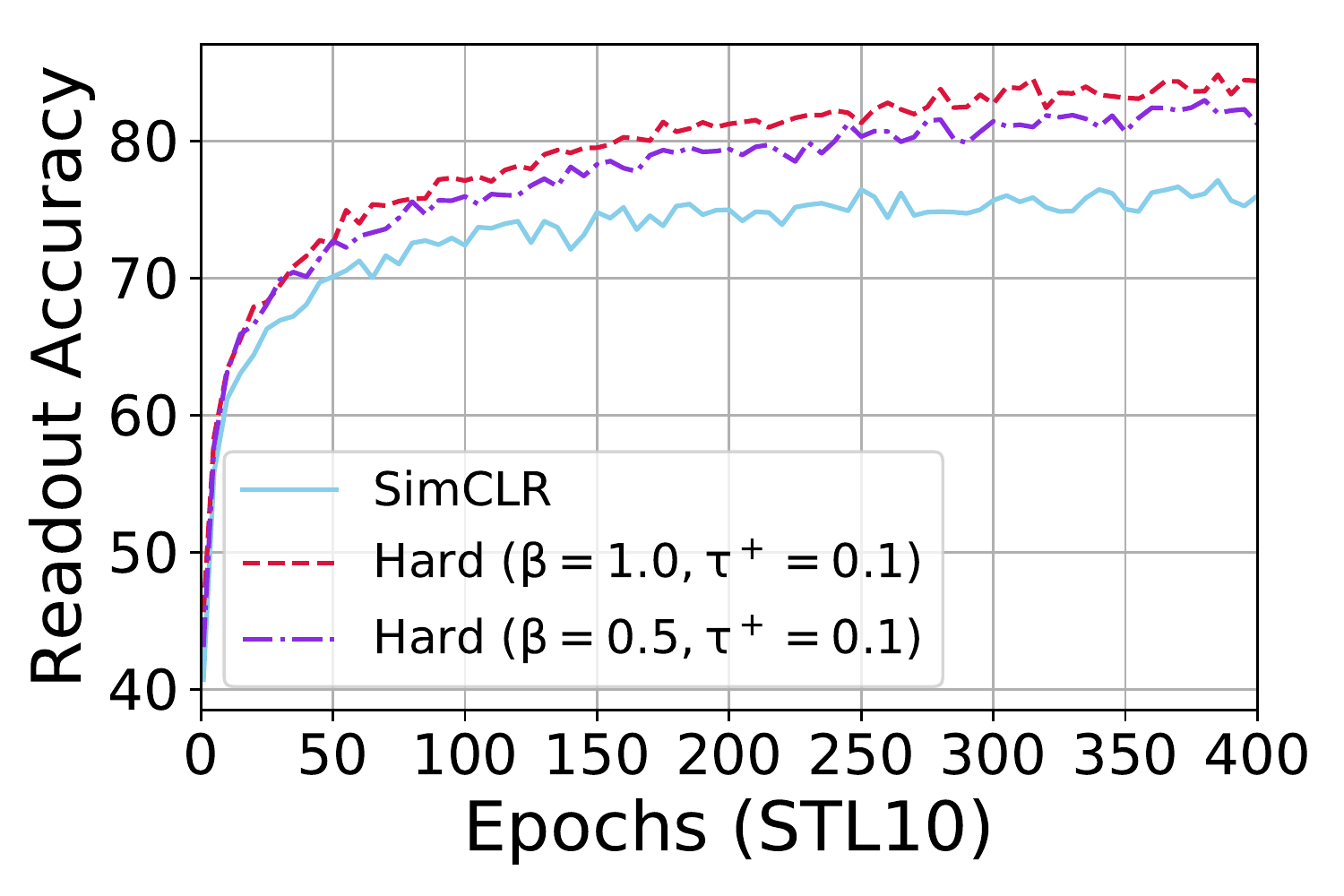} 
     \caption{Hard sampling takes much fewer epochs to reach the same accuracy as SimCLR does in 400 epochs; for STL10 with $\beta=1$ it takes only 60 epochs, and on CIFAR100 it takes 125 epochs (also with $\beta=1$). }
      \label{fig: optimization}
 \end{figure}
 \vspace{-10pt}
 
  \begin{figure}[h]
  \centering
  \includegraphics[width=\textwidth]{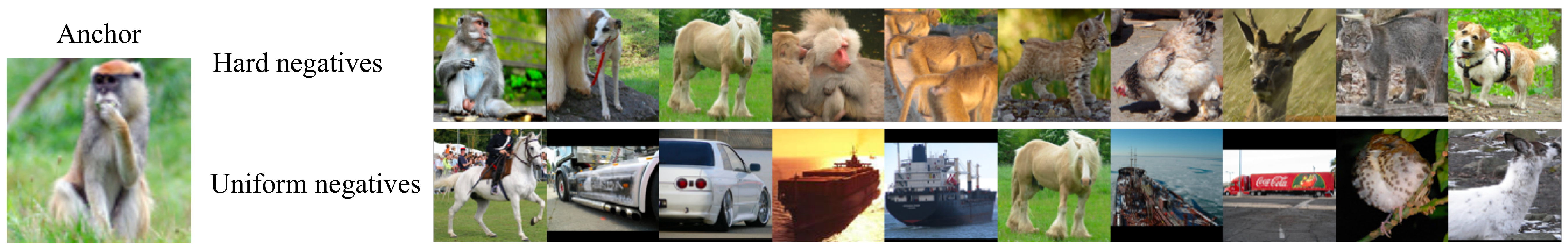}  
     \caption{Qualitative comparison of hard negatives and uniformly sampled negatives for embedding trained on STL10 for $400$ epochs using SimCLR. Top row: selecting the $10$ images with highest inner product with anchor in latent space from a batch of $128$ inputs. Bottom row: a set of random samples from the same batch. Hard negatives are semantically much more similar to the anchor than uniformly sampled negatives - hard negatives possess many similar characteristics to the anchor, including texture, colors, animals vs machinery. }
      \label{fig: real fig1}
 \end{figure}

\section{Experimental Details}\label{ap: experiments}

Figure \ref{fig_objective_code} shows PyTorch-style pseudocode for the standard objective, the debiased objective, and the hard sampling objective. The proposed hard-sample loss is very simple to implement, requiring only two extra lines of code compared to the standard objective. 

\begin{figure*}[t]
\begin{lstlisting}[language=Python]
# pos     : exp of inner products for positive examples
# neg     : exp of inner products for negative examples
# N       : number of negative examples
# t       : temperature scaling
# tau_plus: class probability
# beta    : concentration parameter

#Original objective
standard_loss = -log(pos.sum() / (pos.sum() + neg.sum()))

#Debiased objective
Neg = max((-N*tau_plus*pos + neg).sum() / (1-tau_plus), e**(-1/t))
debiased_loss = -log(pos.sum() / (pos.sum() + Neg))

#Hard sampling objective (Ours)
reweight = (beta*neg) / neg.mean()
Neg = max((-N*tau_plus*pos + reweight*neg).sum() / (1-tau_plus), e**(-1/t))
hard_loss = -log( pos.sum() / (pos.sum() + Neg))
\end{lstlisting}
\caption{Pseudocode for our proposed new hard sample objective, as well as the original NCE contrastive objective, and debiased contrastive objective. In each case we take the number of positive samples to be $M=1$. The implementation of our hard sampling method only requires two additional lines of code compared to the standard objective.} 
\label{fig_objective_code}
\end{figure*}

\subsection{Visual Representations}\label{vision experiments}

We implement SimCLR in PyTorch. We use a ResNet-50 \citep{he2016deep} as the backbone with embedding dimension $2048$ (the representation used for linear readout), and projection head into the lower $128$-dimensional space (the embedding used in the contrastive objective). We use the Adam optimizer \citep{kingma2014adam} with learning rate $0.001$ and weight decay $10^{-6}$. Code available at \url{https://github.com/joshr17/HCL}. Since we adopt the SimCLR framework, the number of negative samples $N=2(\text{batch size} - 1)$. Since we always take the batch size to be a power of 2 ($16,32,64,128,256$) the negative batch sizes are $30,62,126,254,510$ respectively. Unless otherwise stated, all models are trained for $400$ epochs.

\paragraph{Annealing $\beta$ Method:}

We detail the annealing method whose results are given in Figure \ref{fig:stl ablation}. The idea is to reduce the concentration parameter down to zero as training progresses. Specifically, suppose we have $e$ number of total training epochs. We also specify a number $\ell$ of ``changes'' to the concentration parameter  we shall make. We initialize the concentration parameter $\beta_1=\beta$ (where this $\beta$ is the number reported in Figure \ref{fig:stl ablation}), then once every $e/\ell$ epochs we reduce $\beta_i$ by $\beta/\ell$. In other words, if we are currently on $\beta_i$, then $\beta_{i+1} = \beta_i - \beta/\ell$, and we switch from $\beta_i$ to $\beta_{i+1}$ in epoch number $i \cdot e / \ell$. The idea of this method is to select particularly difficult negative samples early on order to obtain useful gradient information early on, but later (once the embedding is already quite good) we reduce the ``hardness'' level so as to reduce the harmful effect of only approximately correcting for false negatives (negatives with the same labels as the anchor).

We also found the annealing in the opposite direction (``down'') achieved similar performance. 

\paragraph{Bias-variance of empirical estimates in hard-negative objective:}

Recall the final hard negative samples objective we derive is,

\begin{equation}
 \mathbb{E}_{\substack{x \sim p \\ x^+ \sim p_x^+ }} \left [-\log \frac{e^{f(x)^T f(x^+)}}{e^{f(x)^T f(x^+)} + \frac{Q}{\tau^-}(\mathbb{E}_{x^- \sim q_\beta} [e^{f(x)^T f(x^-)}] - \tau^+ \mathbb{E}_{v \sim q_\beta^+} [e^{f(x)^T f(v) }])} \right ]. 
\end{equation}

This objective admits a practical counterpart by using empirical approximations to $\mathbb{E}_{x^- \sim q_\beta} [e^{f(x)^T f(x^-)}]$ and $\mathbb{E}_{v \sim q_\beta^+} [e^{f(x)^T f(v) }]$. In practice we use a fairly large number of samples (e.g. $N=510$) to approximate the first expectation, and only $M=1$ samples to approximate the second. Clearly in both cases the resulting estimator is unbiased. Further, since the first expectation is approximated using many samples, and the integrand is bounded, the resulting estimator is well concentrated (e.g. apply Hoeffding's inequality out-of-the-box). But what about the second expectation? This might seem uncontrolled since we use only one sample, however it turns out that the random variable $X = e^{f(x)^T f(v)}$  where $x \sim p$ and $v \sim q^+_\beta$ has variance that is bounded by $ \mathcal L_\text{align}(f) $.

\begin{lemma}
\label{lemma: bias/var}
Consider the random variable $X = e^{f(x)^T f(v)}$ where $x\sim p$ and $v \sim q^+_\beta$. Then $\text{Var}(X) \leq \mathcal O \big ( \mathcal L_\text{align}(f) \big ) $.
\end{lemma}

Recall that $ \mathcal L _\text{align}(f)  = \mathbb{E}_{x,x^+ } \| f(x)-f(x^+)\|^2 /2$ is termed \emph{alignment}, and \cite{wang2018understanding} show that the contrastive objective jointly optimize \emph{alignment}  and \emph{uniformity}. Lemma \ref{lemma: bias/var} therefore shows that as training evolves, the variance of the $X = e^{f(x)^T f(v)}$ where $x \sim p$ and $v \sim q^+_\beta$ is bounded by a term that we expect to see becoming small, suggesting that using a single sample ($M=1$) to approximate this expectation is not unreasonable. We cannot, however, say more than this since we have no guarantee that $\mathcal L _\text{align}(f) $ goes to zero.

\begin{proof}
Fix an $x$ and recall that we are considering $q^+_\beta (\cdot )= q^+_\beta(\cdot ; x)$. First let $X'$ be an i.i.d. copy of $X$, and note that, conditioning on $x$, we have $2\text{Var}(X | x ) = \text{Var}(X| x)+\text{Var}(X'| x) = \text{Var}(X-X'| x) \leq \mathbb{E}\big [ (X-X')^2  | x \big]$. Bounding this difference,

\begin{align*}
\mathbb{E}\big [ (X-X')^2  | x \big] &= \mathbb{E}_{v,v'\sim q^+_\beta} \bigg ( e^{f(x)^\top f(v)} - e^{f(x)^\top f(v')} \bigg ) ^2  \\
&\leq \mathbb{E}_{v,v'\sim q^+_\beta} \bigg ( e^{1/t^2} \big [ {f(x)^\top f(v)} -{f(x)^\top f(v')}\big ] \bigg ) ^2  \\ 
&\leq e^{1/t^4}  \mathbb{E}_{v,v'\sim q^+_\beta} \bigg (\big [ \norm{f(x)}\norm{f(v) - f(v')}\big ] \bigg ) ^2  \\ 
&= \frac{e^{1/t^4}}{t^2}  \mathbb{E}_{v,v'\sim q^+_\beta}\norm{f(v) - f(v')}^2  \\ 
&\leq \mathcal O \bigg ( \mathbb{E}_{v,v'\sim p^+}\norm{f(v) - f(v')}^2  \bigg ) \\ 
\end{align*}

where the first inequality follows since $f$ lies on the sphere of radius $1/t$, the second inequality by Cauchy–Schwarz, the third again since $f$ lies on the sphere of radius $1/t$, and the fourth since $ q^+_\beta$ is absolutely continuous with respect to $p^+$ with bounded ratio.

Since $p^+(x^+)=p(x^+|h(x))$ only depends on $c=h(x)$, rather than $x$ itself, taking expectations over $x\sim p$ is equivalent to taking expectations over $c \sim \rho$. Further, $\rho(c) p(v|c)p(v'|c) = p(v)p(v'|c)=p(v)p^+_{v}(v')$. So $\mathbb{E}_{c \sim \rho} \mathbb{E}_{v,v'\sim p^+}\norm{f(v) - f(v')}^2 = \mathbb{E}_{x,x^+}\norm{f(x) - f(x^+)}^2 = 2 \mathcal L _\text{align}(f)  $, where $x \sim p$ and $x^+ \sim p^+_x$. Thus we obtain the lemma.
\end{proof}

\begin{figure*}[h]
\begin{lstlisting}[language=Python]
# pos     : exp of inner products for positive examples
# neg     : exp of inner products for negative examples
# N       : number of negative examples
# t       : temperature scaling
# tau_plus: class probability
# beta    : concentration parameter

#Clipping negatives trick before computing reweighting 
reweight = 2*neg / max( neg.max().abs(), neg.min().abs() )
reweight = (beta*reweight) / reweight.mean()
Neg = max((-N*tau_plus*pos + reweight*neg).sum() / (1-tau_plus), e**(-1/t))
hard_loss = -log( pos.sum() / (pos.sum() + Neg))
\end{lstlisting}
\caption{In cases where the learned embedding is not normalized to lie on a hypersphere we found that clipping the negatives to live in a fixed range (in this case $[-2,2]$)  stabilizes optimization.} 
\label{fig_trick}
\end{figure*}

\subsection{Graph Representations}

All datasets we benchmark on can be downloaded at  \url{www.graphlearning.io} from the TUDataset repository of graph classification problems \citep{morris2020tudataset}. Information on basic statistics of the datasets is included in Tables \ref{graph data: 1} and \ref{graph data: 2}. For fair comparison to the original InfoGraph method, we adopt the official code, which can be found at  \url{https://github.com/fanyun-sun/InfoGraph}. We modify only the \texttt{gan\_losses.py} script, adding in our proposed hard sampling via reweighting. For simplicity we trained all models using the same set of hyperparameters: we used the GIN architecture  \citep{xu2018powerful} with $K=3$ layers and embedding dimension $d=32$. Each model is trained for $200$ epochs with batch size $128$ using the Adam optimizer \citep{kingma2014adam}. with learning rate $0.001$, and weight decay of $10^{-6}$.
Each embedding is evaluated using the average accuracy 10-fold cross-validation using an SVM as the classifier (in line with the approach taken by \cite{morris2020tudataset}). Each experiment is repeated from scratch 10 times, and the distribution of results from these 10 runs is plotted in Figure \ref{graph_table}. 

Since the graph embeddings are not constrained to lie on a hypersphere, for a batch we clip all the inner products to live in the interval $[-2,2]$ while computing the reweighting, as illustrated in Figure \ref{fig_trick}. We found this to be important for stabilizing optimization.

\begin{table*}[!ht]
\small
\begin{center}
\begin{tabularx}{0.58\textwidth}{l| c*{4}{c}}
\hline
 \fontsize{9pt}{9pt}\selectfont\textbf{Dataset} &  \fontsize{9pt}{9pt}\selectfont\textbf{DD} & \fontsize{9pt}{9pt}\selectfont\textbf{PTC} & \fontsize{9pt}{9pt}\selectfont\textbf{REDDIT-B} & \fontsize{9pt}{9pt}\selectfont\textbf{PROTEINS} & \fontsize{9pt}{9pt}
 \\
\hline
\hline
No. graphs & 1178 & 344 & 2000 & 	1113 \\
No. classes & 2 & 2 & 2 & 2 \\
Avg. nodes & 284.32 & 14.29 & 429.63 & 39.06 \\
Avg. Edges & 715.66 & 14.69 & 497.75 & 72.82 \\
\end{tabularx}
\end{center}
\caption{Basic statistics for graph datasets.}
\label{graph data: 1}
\end{table*}

\begin{table*}[!ht]
\small
\begin{center}
\begin{tabularx}{0.615\textwidth}{l| c*{4}{c}}
\hline
 \fontsize{9pt}{9pt}\selectfont\textbf{Dataset} &  \fontsize{9pt}{9pt}\selectfont\textbf{ENZYMES} & \fontsize{9pt}{9pt}\selectfont\textbf{MUTAG} & \fontsize{9pt}{9pt}\selectfont\textbf{IMDB-B} & \fontsize{9pt}{9pt}\selectfont\textbf{IMDB-M} & \fontsize{9pt}{9pt}
 \\
\hline
\hline
No. graphs & 600 & 	188 & 1000 & 1500 \\
No. classes & 6 & 2 & 2 & 3 \\
Avg. nodes & 	32.63 & 	17.93 & 19.77 & 13.00 \\
Avg. Edges & 62.14 & 19.79 & 96.53 & 65.94 \\
\end{tabularx}
\end{center}
\caption{Basic statistics for graph datasets.}
\label{graph data: 2}
\end{table*}

\subsection{Sentence Representations}

We adopt the official quick-thoughts vectors experimental settings, which can be found at  \url{https://github.com/lajanugen/S2V}. We keep all hyperparameters at the default values and change only the  \texttt{s2v-model.py} script. Since the official BookCorpus dataset \cite{kiros2015skip} is not available, we use an unofficial version obtained using the following repository: \url{https://github.com/soskek/bookcorpus}. Since the sentence embeddings are also not constrained to lie on a hypersphere, we use the same clipping trick as for the graph embeddings, illustrated in Figure \ref{fig_trick}. 

After training on the BookCorpus dataset, we evaluate the embeddings on six different classification tasks: paraphrase identification (MSRP) \citep{dolan2004unsupervised},  question type classification (TREC) \citep{voorhees2002overview}, opinion polarity (MPQA) \citep{wiebe2005annotating}, subjectivity classification (SUBJ) \citep{pang2004sentimental}, product reviews (CR) \citep{hu2004mining}, and sentiment of movie reviews (MR)  \citep{pang2005seeing}.

\paragraph{Comparison with \cite{kalantidis2020hard}:}

 \cite{kalantidis2020hard} also consider ways to sample negatives, and propose a mixing strategy for hard negatives, called MoCHi. The main points of difference are: 1) MoCHi considers the benefit of hard negatives, but does not consider the possibility of false negatives (Principle 1), which we found to be valuable. 2) MoCHi introduces three extra hyperparameters, while our method introduces only two ($\beta$, $\tau^+$). If we discard Principle 1 (i.e. $\tau^+$) then only $\beta$ requires tuning. 3) our method introduces zero computational overhead by utilizing within-batch reweighting, whereas MoCHi involves a small amount of extra computation.